\documentclass{article}
\usepackage[letterpaper, left=1in, right=1in, top=1in, bottom=1in]{geometry}
\usepackage{natbib}
\bibpunct{(}{)}{;}{a}{,}{,}
\usepackage[colorlinks,citecolor=blue!70!black,linkcolor=blue!70!black,
urlcolor=blue!70!black,breaklinks=true]{hyperref}
\usepackage{parskip}
\usepackage[dvipsnames,table]{xcolor}
\usepackage{microtype}
\usepackage{booktabs} 
\usepackage{makecell}
\usepackage{amsthm}
\usepackage{amsmath}
\usepackage{dsfont} 
\usepackage{amssymb}
\usepackage{xspace}
\usepackage{url}            
\usepackage{algorithm}
\usepackage[noend]{algorithmic}
\usepackage{listings}
\usepackage{bm}
\usepackage{bbm}
\usepackage{enumitem}
\usepackage{nicefrac}       
\usepackage{mathrsfs} 
\usepackage{inconsolata}
\usepackage[subrefformat=parens,labelformat=parens]{subfig}
\usepackage{wrapfig}
\usepackage{graphicx}
\usepackage{lipsum}
\usepackage{bigints}
\usepackage{transparent}
\usepackage{comment}
\usepackage[nameinlink,capitalize]{cleveref}
\makeatletter
\AddToHook{cmd/appendix/before}{
\def\cref@section@alias{appendix}
\def\cref@subsection@alias{appendix}
\def\cref@subsubsection@alias{appendix}
}
\makeatother
\usepackage{thmtools}
\usepackage{thm-restate}
\usepackage{nicematrix}
\usepackage{arydshln}
\usepackage{fancyhdr}
\usepackage{mathtools}
\usepackage[scaled=.88]{helvet}
\usepackage{breakcites}
\usepackage{multirow}
\usepackage{array}

\usepackage{colortbl}

\usepackage{enumitem} %
\usepackage{relsize} %

\DeclareMathOperator*{\argmax}{arg\,max}
\DeclareMathOperator*{\argmin}{arg\,min}

\renewcommand{\epsilon}{\varepsilon}

\newtheoremstyle{spaced}
  {6pt}   %
  {0pt}   %
  {\itshape} %
  {}       %
  {\bfseries} %
  {.}      %
  {0.5em}  %
  {}
\theoremstyle{spaced}

\newcommand{\algcommentlight}[1]{\textcolor{blue!70!black}{\transparent{0.5}\small{\texttt{\textbf{//\hspace{2pt}#1}}}}}

\DeclarePairedDelimiter{\abs}{\lvert}{\rvert} %

\DeclarePairedDelimiter{\crl}{\{}{\}}
\DeclarePairedDelimiter{\prn}{(}{)}

\DeclarePairedDelimiter{\ang}{\langle}{\rangle}

\DeclarePairedDelimiterX{\infdiv}[2]{(}{)}{%
  #1\;\delimsize\|\;#2%
}

\newcommand{\wb}[1]{\widebar{#1}}

\def\ddefloop#1{\ifx\ddefloop#1\else\ddef{#1}\expandafter\ddefloop\fi}
\def\ddef#1{\expandafter\def\csname bb#1\endcsname{\ensuremath{\mathbb{#1}}}}
\ddefloop ABCDEFGHIJKLMNOPQRSTUVWXYZ\ddefloop
\def\ddefloop#1{\ifx\ddefloop#1\else\ddef{#1}\expandafter\ddefloop\fi}
\def\ddef#1{\expandafter\def\csname b#1\endcsname{\ensuremath{\mathbf{#1}}}}
\ddefloop ABCDEFGHIJKLMNOPQRSTUVWXYZ\ddefloop
\def\ddef#1{\expandafter\def\csname sf#1\endcsname{\ensuremath{\mathsf{#1}}}}
\ddefloop ABCDEFGHIJKLMNOPQRSTUVWXYZ\ddefloop
\def\ddef#1{\expandafter\def\csname c#1\endcsname{\ensuremath{\mathcal{#1}}}}
\ddefloop ABCDEFGHIJKLMNOPQRSTUVWXYZ\ddefloop
\def\ddef#1{\expandafter\def\csname h#1\endcsname{\ensuremath{\widehat{#1}}}}
\ddefloop ABCDEFGHIJKLMNOPQRSTUVWXYZ\ddefloop
\def\ddef#1{\expandafter\def\csname hc#1\endcsname{\ensuremath{\widehat{\mathcal{#1}}}}}
\ddefloop ABCDEFGHIJKLMNOPQRSTUVWXYZ\ddefloop
\def\ddef#1{\expandafter\def\csname t#1\endcsname{\ensuremath{\widetilde{#1}}}}
\ddefloop ABCDEFGHIJKLMNOPQRSTUVWXYZ\ddefloop
\def\ddef#1{\expandafter\def\csname tc#1\endcsname{\ensuremath{\widetilde{\mathcal{#1}}}}}
\ddefloop ABCDEFGHIJKLMNOPQRSTUVWXYZ\ddefloop
\def\ddefloop#1{\ifx\ddefloop#1\else\ddef{#1}\expandafter\ddefloop\fi}
\def\ddef#1{\expandafter\def\csname scr#1\endcsname{\ensuremath{\mathscr{#1}}}}
\ddefloop ABCDEFGHIJKLMNOPQRSTUVWXYZ\ddefloop

\let\oldparagraph\paragraph
\renewcommand{\paragraph}[1]{\oldparagraph{#1}}

\renewcommand{\epsilon}{\varepsilon}

\newcommand{\ldef}{\vcentcolon=}

\renewcommand{\bigm}[1]{%
  \ifcsname fenced@\string#1\endcsname
    \expandafter\@firstoftwo
  \else
    \expandafter\@secondoftwo
  \fi
  {\expandafter\amsmath@bigm\csname fenced@\string#1\endcsname}%
  {\amsmath@bigm#1}%
}

\newcommand{\DeclareFence}[2]{\@namedef{fenced@\string#1}{#2}}
\makeatother

\makeatletter
\let\save@mathaccent\mathaccent
\newcommand*\if@single[3]{%
  \setbox0\hbox{${\mathaccent"0362{#1}}^H$}%
  \setbox2\hbox{${\mathaccent"0362{\kern0pt#1}}^H$}%
  \ifdim\ht0=\ht2 #3\else #2\fi
  }
\newcommand*\rel@kern[1]{\kern#1\dimexpr\macc@kerna}
\newcommand*\widebar[1]{\@ifnextchar^{{\wide@bar{#1}{0}}}{\wide@bar{#1}{1}}}
\newcommand*\wide@bar[2]{\if@single{#1}{\wide@bar@{#1}{#2}{1}}{\wide@bar@{#1}{#2}{2}}}
\newcommand*\wide@bar@[3]{%
  \begingroup
  \def\mathaccent##1##2{%
    \let\mathaccent\save@mathaccent
    \if#32 \let\macc@nucleus\first@char \fi
    \setbox\z@\hbox{$\macc@style{\macc@nucleus}_{}$}%
    \setbox\tw@\hbox{$\macc@style{\macc@nucleus}{}_{}$}%
    \dimen@\wd\tw@
    \advance\dimen@-\wd\z@
    \divide\dimen@ 3
    \@tempdima\wd\tw@
    \advance\@tempdima-\scriptspace
    \divide\@tempdima 10
    \advance\dimen@-\@tempdima
    \ifdim\dimen@>\z@ \dimen@0pt\fi
    \rel@kern{0.6}\kern-\dimen@
    \if#31
      \overline{\rel@kern{-0.6}\kern\dimen@\macc@nucleus\rel@kern{0.4}\kern\dimen@}%
      \advance\dimen@0.4\dimexpr\macc@kerna
      \let\final@kern#2%
      \ifdim\dimen@<\z@ \let\final@kern1\fi
      \if\final@kern1 \kern-\dimen@\fi
    \else
      \overline{\rel@kern{-0.6}\kern\dimen@#1}%
    \fi
  }%
  \macc@depth\@ne
  \let\math@bgroup\@empty \let\math@egroup\macc@set@skewchar
  \mathsurround\z@ \frozen@everymath{\mathgroup\macc@group\relax}%
  \macc@set@skewchar\relax
  \let\mathaccentV\macc@nested@a
  \if#31
    \macc@nested@a\relax111{#1}%
  \else
    \def\gobble@till@marker##1\endmarker{}%
    \futurelet\first@char\gobble@till@marker#1\endmarker
    \ifcat\noexpand\first@char A\else
      \def\first@char{}%
    \fi
    \macc@nested@a\relax111{\first@char}%
  \fi
  \endgroup
}
\makeatother

\newcommand{\Rand}{\textsc{Rand}\xspace}

\newcommand{\Coreset}{\textsc{Coreset}\xspace}

\newcommand{\Marg}{\textsc{Uncertainty}\xspace}
\newcommand{\badge}{\textsc{Badge}\xspace}

\newcommand{\Ours}{\textsc{Ours}\xspace}

\newcommand{\distmath}{\mathsf{dist}}
\definecolor{cyanblue}{RGB}{0,175,210}
\newcolumntype{?}{!{\vrule width 1.2pt}}

\title{Towards Multimodal Active Learning:\\ Efficient Learning with Limited Paired Data}
\author{
Jiancheng Zhang\\
{\normalsize University of California, Riverside}\\
{\normalsize\texttt{jzhan745@ucr.edu}}
\and
Yinglun Zhu\textsuperscript{\dag}\\
{\normalsize University of California, Riverside}\\
{\normalsize\texttt{yzhu@ucr.edu}}
}

\begin{document}

\maketitle
\begingroup
\renewcommand\thefootnote{}\footnotetext{\textsuperscript{\dag}Project lead and corresponding author.}
\endgroup

\begin{abstract}

Active learning (AL) is a principled strategy to reduce annotation cost in data-hungry deep learning. However, existing AL algorithms focus almost exclusively on unimodal data, overlooking the substantial annotation burden in multimodal learning. We introduce the first framework for \emph{multimodal active learning with unaligned data}, where the learner must actively acquire cross-modal alignments rather than labels on pre-aligned pairs. This setting captures the practical bottleneck in modern multimodal pipelines, where unimodal features are easy to obtain but high-quality alignment is costly. We develop a new algorithm that combines uncertainty and diversity principles in a modality-aware design, achieves linear-time acquisition, and applies seamlessly to both pool-based and streaming-based settings. Extensive experiments on benchmark datasets demonstrate that our approach consistently reduces multimodal annotation cost while preserving performance; for instance, on the ColorSwap dataset it cuts annotation requirements by up to $40\%$ without loss in accuracy.

\end{abstract}

\section{Introduction}
\label{sec:intro}

Deep learning has achieved remarkable success across a wide range of applications, but its effectiveness often hinges on access to large amounts of annotated training data.  
Active learning (AL) has long been viewed as a promising approach to reduce annotation cost by selectively querying the most \emph{informative} instances for labeling \citep{settles2009active}.  
Theoretically, AL can \emph{exponentially} reduce the amount of labeled data required \citep{zhu2022active}, and empirically, it has delivered consistent gains in data efficiency for deep models \citep{sener2018active, Ash2020Deep, citovsky2021batch, saran2023streaming, zhang2024labelbench}, as well as more recently for large language models \citep{margatina2023active, bhatt2024experimental, yuan2024hide}.  

Despite recent progress, most AL methods operate in a \emph{unimodal and unidirectional} setting, where the learner queries class labels for given features inputs.
Multimodal extensions remain limited: \citet{shen2023towards} studies multimodal AL with \emph{pre-aligned} vision-language pairs, effectively reducing the task to unidirectional AL on composite inputs. 
\citet{zhang2024learnability} applies AL to a specific multimodal task---video captioning---which similarly treats one modality as input and the other as annotation.
As a result, existing approaches sidestep the core challenge of actively discovering cross-modal correspondences in unaligned data.
\looseness=-1

\looseness=-1

In this paper, we introduce the first setting of \emph{multimodal active learning with unaligned data}, where the learner begins with independent vision and language features and must actively acquire cross-modal alignments.  
Unlike unimodal AL or pre-aligned multimodal AL, our setting requires deciding both \emph{which modality to query from} and \emph{how to align instances across modalities}.  
This formulation raises two qualitatively new challenges:  
(i) \emph{bidirectional alignment}, since annotation can begin from either vision-to-language or language-to-vision, and different choices lead to distinct annotation sets and learning trajectories; and  
(ii) a \emph{large cross-modal candidate space}, since evaluating the utility of an instance requires reasoning over potential matches across the entire other modality, which may contain millions of unique candidates \citep{gadre2024datacomp}.  
Naively scoring all image–text pairs scales quadratically, making classical AL strategies computationally infeasible.  

Our setting is directly motivated by modern multimodal pipelines, where raw modality-specific features can be obtained cheaply at scale, but high-quality alignment is expensive, domain-specific, and often the true bottleneck \citep{gadre2024datacomp, bai2024survey}.  
This challenge is especially acute in specialized domains (or when adapting pretrained models to such domains), such as medical imaging \citep{chen2024alifuse} and autonomous driving \citep{ge2023metabev}, where multimodal annotation is both costly and essential.

\paragraph{Our contributions.}  
Our main contributions are as follows:  
\begin{enumerate}[label=(\roman*)]  
\item We introduce the problem of \emph{multimodal active learning with unaligned data}, clarifying how it differs fundamentally from unimodal AL and existing multimodal AL with pre-aligned data.  
\item We develop a new algorithm that integrates uncertainty and diversity principles in a modality-aware design, achieves \emph{linear-time} complexity in the number of unaligned instances, and applies seamlessly to both \emph{pool-based} and \emph{streaming-based} scenarios.  
\item We conduct extensive experiments on benchmark datasets, demonstrating consistent annotation savings (up to $40\%$ on ColorSwap) while maintaining competitive performance.  
\end{enumerate}

\paragraph{Paper organization.}  
The remainder of the paper is organized as follows.
\Cref{sec:related} reviews related work.
\Cref{sec:setting} introduces our problem formulation and highlights the unique challenges of multimodal active learning with unaligned data.
\Cref{sec:methods} presents our algorithm, together with its complexity analysis and extensions.
\Cref{sec:exs} reports the main experimental results, and \cref{sec:discussion} provides additional analysis.
\Cref{sec:conclusion} concludes the paper.
Proofs, implementation details, and supplementary experiments are deferred to the Appendix.
\looseness=-1

\section{Related Work}
\label{sec:related}

\paragraph{Active Learning.}
Active learning (AL) studies how to train accurate models with fewer annotations by selectively querying informative samples \citep{settles2009active}.
A substantial body of theoretical work establishes provable advantages of AL over passive learning under various assumptions
\citep{castro2007minimax, balcan2007margin, dasgupta2009analysis, hanneke2014theory, zhang2014beyond, krishnamurthy2019active, puchkin2021exponential, zhu2022efficient, zhu2022active}.
On the empirical side, AL has been shown to improve data efficiency in deep learning, including batch-mode methods and streaming settings
\citep{sener2018active, Ash2020Deep, citovsky2021batch, ash2021gone, wang2022deep, saran2023streaming,zhang2024labelbench},
and more recently for large pretrained models \citep{margatina2023active, bhatt2024experimental, yuan2024hide}.
\looseness=-1

Most existing AL methods focus on the \emph{unimodal} setting with \emph{unidirectional annotation}: 
given unlabeled feature representations, the learner selectively queries class labels from a small, fixed label set.
Multimodal AL is far less explored.
The most relevant effort is the \emph{pre-aligned} multimodal AL framework of \citet{shen2023towards},
which \emph{assumes} that vision-language pairs are \emph{already aligned} and therefore restricts queries to class labels for these aligned pairs.  
Since the alignment is free, the problem effectively reduces to unidirectional AL on composite inputs rather than addressing the core challenge of discovering cross-modal correspondences.  
Other applications of AL to multimodal tasks, such as video captioning \citep{zhang2024learnability}, follow a similar path: one modality (e.g., video) is treated as input while the other (e.g., text) serves as annotation, keeping the process strictly unidirectional and structurally unimodal.

In contrast, we introduce the first multimodal AL setting that supports \emph{bidirectional alignment} with unaligned data: the learner is provided with independent vision and language features and must actively acquire meaningful cross-modal correspondences, either from images to text or from text to images.  
This setting is directly motivated by modern multimodal pipelines such as CLIP \citep{radford2021learning} and SigLIP \citep{zhai2023sigmoid}, where unimodal features are easy to obtain at scale, but high-quality alignment is expensive, domain-specific, and often the true bottleneck.

\paragraph{Multimodal learning.}  
Multimodal learning seeks to integrate information from diverse modalities such as text, images, and audio to improve learning performance \citep{baltruvsaitis2018multimodal, liang2024foundations}.  
Early approaches relied on supervised labels, where multimodal features were annotated with classification labels, making the process closely resemble standard unimodal learning.  
More recently, multimodal learning has shifted toward supervision from paired multimodal data, where one modality provides a supervision signal for another \citep{zong2024self}.  
A prominent example is CLIP \citep{radford2021learning} and its variants \citep{zhai2022lit, zhai2023sigmoid}, which leverage paired image-text data to train contrastive objectives that align representations across modalities. 

Many recent works study problems related to multimodal alignment but with goals that differ from active learning with unaligned data. For example, \citet{han2024learning} focuses on correcting mismatched cross-modal pairs under the assumption that most pairs are already correctly aligned, while \citet{kimura2025infomae} studies aligning independently pretrained unimodal encoders without performing cross-modal data alignment.

While large-scale, noisily aligned multimodal data can be scraped from the web, it is increasingly recognized that training high-performing multimodal models requires high-quality, well-aligned datasets \citep{gadre2024datacomp, bai2024survey}.  
This need is even more pronounced in specialized domains such as medical imaging \citep{chen2024alifuse} and autonomous driving \citep{ge2023metabev}, where careful multimodal annotation is both critical and costly.  
These challenges highlight the importance of developing efficient methods that can learn effectively from fewer paired examples.  
Although recent work has explored active learning for multimodal tasks with label annotations \citep{shen2023towards}, to the best of our knowledge, our work is the first to design a multimodal active learning algorithm specifically tailored for \emph{pairing annotation}.

\paragraph{Data selection.}  
Data selection is closely related to active learning, aiming to construct a high-quality subset of data for more efficient or effective model training.  
The key distinction lies in the availability of labels or pairings: active learning selects data points \emph{before} annotation, whereas data selection assumes a fully labeled or paired dataset.  
From this perspective, active learning is strictly more challenging, as it must operate without access to labeling or pairing information.  

Data selection methods have been shown to reduce training cost \citep{schreiber2020apricot, mindermann2022prioritized, sorscher2022beyond, yang2023dataset, shen2024efficient}, and in some cases even improve performance by removing duplicated or noisy data \citep{lee2022deduplicating, tirumala2023d4, xia2024less}.  
In the multimodal domain, researchers have proposed metrics such as CLIPScore \citep{hessel2021clipscore} and its extensions \citep{wang2024finetuned, wang2024cliploss, joshi2024data} to evaluate the quality of image-ext pairs, enabling filtering of low-quality examples for data selection.  
While CLIPScore-based filtering has proven useful for multimodal data selection \citep{schuhmann2021laion}, it assumes access to pre-paired multimodal data and is therefore unsuitable for our setting with \emph{unaligned} modalities.

\section{Problem Setting}
\label{sec:setting}

We study multimodal learning with a dataset $\cD = \prn{\cD^{v}, \cD^{l}}$, where $\cD^{v} = \crl{x_i^{v}}_{i=1}^{n}$ denotes the collection of raw vision features and $\cD^l = \crl{x_i^l}_{i=1}^n$ denotes the collection of raw textual/language features.\footnote{For simplicity, we focus on the vision-language case. Our setting and algorithms naturally extend to general multimodal learning with more than two modalities; see \cref{sec:extensions}.}  
Unlike standard multimodal setups, \emph{the vision and language features are initially unaligned}.  
The learner may query a subset of instances to obtain their aligned pairs at an annotation cost.  
Specifically, for any data point $x^k_i \in \crl{x_i^v, x_i^l}$, the learner can spend one unit of \emph{annotation cost} to reveal its aligned pair $x_i \ldef \prn{x_i^v, x_i^l}$.  
We use $\cS = \crl{\prn{x_i^{v}, x_i^{l}}}_{i=1}^m$ to denote the set of annotated pairs obtained with a total of $m$ units of cost.  

The goal, under a fixed annotation budget, is to \emph{actively and strategically} select an informative subset $\cS$ to maximize the quality of a multimodal model $\phi \ldef \prn{\phi^v, \phi^l}$, where $\phi^v, \phi^l: \bbR^{d^\prime} \rightarrow \bbR^d$ are encoders that map raw features into a shared representation space.  
We adopt CLIP-style contrastive training for multimodal models \citep{radford2021learning}, and, following standard practice \citep{zhai2022lit, zhai2023sigmoid}, evaluate model quality on downstream tasks.  

We refer to this setup as \emph{multimodal active learning with unaligned data}, which not only extends classical unimodal active learning \citep{sener2018active, Ash2020Deep, citovsky2021batch, saran2023streaming} but also departs fundamentally from prior multimodal active learning frameworks that assume \emph{pre-aligned} data \citep{shen2023towards}.
The active learning algorithm proceeds over $T \in \bbZ_+$ iterations.  
At iteration $t$, the learner selects and annotates a batch of $B$ data points $\crl{\prn{x_{t_i}^v, x_{t_i}^l}}_{i=1}^B$, and updates the annotation set as $\cS_t \gets \cS_{t-1} \cup \crl{\prn{x_{t_i}^v, x_{t_i}^l}}_{i=1}^B$.  
The multimodal model $\phi_t = \prn{\phi_t^v, \phi_t^l}$ is trained on $\cS_t$ and then used to guide data selection in the next iteration.  
This iterative process enables the learner, under a fixed budget, to build a high-quality multimodal model from strategically chosen alignments.  

Depending on how the learner accesses the unaligned pool, we study two regimes:
\begin{itemize}[leftmargin=10pt, itemindent=*]
  \item \textbf{Pool-based multimodal active learning.}  
  The learner has full access to $\cD = \prn{\cD^v, \cD^l}$ throughout the process and can query any instances at any $t \in [T]$ \citep{Ash2020Deep}.
  \item \textbf{Streaming-based multimodal active learning.}  
  The data arrives as disjoint subsets $\cD = \crl{\cD_t}_{t=1}^T$ with $\cup_{t=1}^T \cD_t = \cD$ and $\cD_{t_i} \cap \cD_{t_j} = \emptyset$.  
  At iteration $t$, the learner only observes $\cD_t = \prn{\cD_t^v, \cD_t^l}$ and must select and annotate data \emph{within} this batch.  
  Unqueried data from $\cD_t$ cannot be revisited in future iterations \citep{saran2023streaming}.
\end{itemize}

\paragraph{Additional notation.}
For any $N \in \bbZ_+$, we denote $[N] \ldef \crl{1, \cdots, N}$.  
For multimodal sets $\cD = \crl{\cD^v, \cD^l}$ and $\cS = \crl{\cS^v, \cS^l}$, we write $\cD \setminus \cS \ldef \crl{\cD^v \setminus \cS^v, \cD^l \setminus \cS^l}$.  
For any modality $k \in \crl{v, l}$, we denote $\phi^k_t(\cS_t^k) \ldef \crl{\phi_t^k(x^k): x^k \in \cS_t^k}$.  
When clear, we use the shorthand $\phi_t(\cS_t^k) = \phi^k_t(\cS_t^k)$ or $\phi(\cS_t^k) = \phi^k_t(\cS_t^k)$. 

\subsection{Unique Challenges with Unaligned Multimodal Data}
\label{sec:unique}

Compared to unimodal active learning \citep{sener2018active, Ash2020Deep} and multimodal active learning with \emph{pre-aligned} data \citep{shen2023towards}, our setting---multimodal active learning with \emph{unaligned} data---poses qualitatively new challenges.  

In unimodal active learning, each instance $x$ has a single feature vector and annotation assigns a class label from a small predefined set.  
In multimodal active learning with \emph{pre-aligned} data, the learner selects from pre-aligned pairs $\crl{(x_i^v, x_i^l)}_{i=1}^n$ for label queries.  
Because modalities are already aligned, this effectively reduces to unimodal active learning on composite inputs.  

By contrast, in our unaligned setting with vision features $\{x_i^v\}_{i=1}^n$ and language features $\{x_i^l\}_{i=1}^n$, the learner must simultaneously decide \emph{which modality to query from} and \emph{how to align instances across modalities}.  
This creates two distinctive challenges:  

\begin{itemize}[leftmargin=10pt, itemindent=*]
  \item \textbf{Bidirectional alignment.}  
  With unaligned data, annotation may begin from either vision-to-language or language-to-vision.  
  Crucially, the learner does not know in advance which instance from the other modality will be paired.  
  Different alignment directions can lead to entirely different annotation sets, adding an extra decision layer absent in unimodal or pre-aligned multimodal AL.  
  \looseness=-1

  \item \textbf{Large cross-modal candidate space.}  
  Even after choosing a modality to query, the utility of a candidate must be evaluated against the entire other modality.  
  For instance, querying an image requires scoring potential matches across all texts, which effectively serves as an enormous candidate label space.  
  Unlike conventional class labels, these candidates are instance-specific and extremely numerous (e.g., 12.8M unique texts in DataComp \citep{gadre2024datacomp}).  
  Naively evaluating all pairs scales quadratically with dataset size, quickly becoming infeasible.  
\end{itemize}

Thus, multimodal active learning with \emph{unaligned} data requires acquisition algorithms that handle both bidirectional alignment and the large cross-modal search space efficiently.  
We present our approach to these challenges in \cref{sec:methods}.

\section{Methods}
\label{sec:methods}
We present our approaches to address the unique challenges mentioned in \cref{sec:unique}. 
We first introduce our multimodal active learning algorithm for the pool-based setting in \cref{sec:multimodal_AL}, which is further extended to the streaming-based setting and to settings beyond vision-language models in \cref{sec:extensions}.

\subsection{Multimodal Active Learning}
\label{sec:multimodal_AL}

We present our multimodal active learning algorithm in \cref{alg:multimodal_AL}, 
which proceeds iteratively for $T$ iterations. 
At each iteration $t \in [T]$, \cref{alg:multimodal_AL} selects a batch of $B$ data points for annotation, based on the multimodal model $\phi_{t-1} = \prn{\phi^v_{t-1}, \phi_{t-1}^l}$ trained with respect to previously annotated data points $\cS_{t-1}$. 
The annotation set is then updated to $\cS_t$,  
and the model $\phi_t = (\phi_t^v, \phi_t^l)$ is retrained on the updated annotation dataset to guide data selection in the next iteration.

The core of \cref{alg:multimodal_AL} lies in how to actively select data for annotation. This is achieved via three integrated steps: (1) modality selection, (2) coreset construction, and (3) uncertainty-based selection.
At each iteration, \cref{alg:multimodal_AL} first selects a modality that is \emph{underrepresented} by the already annotated data $\cS_{t-1}$ (Step 1). It then constructs a coreset of $B_C$ data on the selected modality to ensure \emph{coverage} (Step 2), and finally selects $B \leq B_C$ highly \emph{uncertain} data points from this coreset using a cross-modal uncertainty score (Step 3), 
which quantifies how confidently a feature in one modality matches candidates from the other.

 \begin{algorithm}[htbp]
	\caption{Multimodal Active Learning}
    \label{alg:multimodal_AL}
	\renewcommand{\algorithmicrequire}{\textbf{Input:}}	\renewcommand{\algorithmicensure}{\textbf{Output:}}
	\newcommand{\algorithmicbreak}{\textbf{break}}
    \newcommand{\BREAK}{\STATE \algorithmicbreak}
	\begin{algorithmic}[1]
\REQUIRE 
    Unaligned multimodal dataset $\cD = \crl{\cD^v, \cD^l}$, number of iterations $T$, per-round selection size $B$, coreset hyperparameter $B_C \geq B$.
    \STATE Initialize multimodal model $\phi_0 = \crl{\phi_0^v, \phi_0^l}$ with random or pretrained weights.
    \STATE Initialize the annotation set $\cS_0 = \emptyset$. \label{line:initialize_set}
\FOR{$t = 1, \cdots, T$}
    \STATE Consider unaligned data pool $\cD_t \ldef \cD \setminus \cS_{t-1}$.
    \label{line:unaligned_data}
    \STATE \textbf{Step 1: Modality selection.}
    \label{line:modality_selection}
    Select modality $k_t \ldef \argmax_{k \in \crl{v, l}} d_t^k$ that is less-covered by previous annotations $\cS_{t-1}$, where  
\begin{equation}
\label{eq:maxmin_dist}
  d^k_t \ldef \max_{z_i \in \wb \phi_{t-1}(\cD_{t}^{k_t})} 
  \, \min_{z_j \in \wb \phi_{t-1}(\cS^{k_t}_{t-1}) } 
  \distmath (z_i, z_j).
\end{equation}
    \STATE \textbf{Step 2: Coreset construction.} On modality $k_t$, construct a coreset $\cC_t^{k_t} \subseteq \cD_t^{k_t}$ of size $B_C$ such that, together with $\cS_{t-1}^{k_t}$, it maximally covers $\cD_t^{k_t}$: \label{line:coreset}
\begin{equation}
\label{eq:coreset}
  \cC_t^{k_t} \ldef \argmin_{\cC: \abs{\cC} = B_C}\,  \max_{z_i \in \phi_{t-1}(\cD_t^{k_t} \setminus \cC)} \, \min_{z_j \in \phi_{t-1}(\cS_{t-1}^{k_t} \cup \cC)} \distmath(z_i, z_j).
\end{equation}
    \algcommentlight{\cref{eq:coreset} can be approximated with an efficient greedy algorithm (\cref{alg:k-center}).}

    \STATE \textbf{Step 3: Uncertainty-based selection.}
    Within coreset $\cC_t^{k_t}$, select the top-$B$ \emph{most uncertain} data points using multimodal model $\phi_{t-1}$.  \label{line:step3}
\STATE    Let $m \ldef \crl{v, l} \setminus \crl{k_t}$ denote the unselected modality in Step 1. 
    \STATE For each data $x_i^{k_t} \in \cC_t^{k_t}$, compute its margin score $u(x_i^{k_t}) \ldef w^{i}_{(1)} - w^i_{(2)}$, which serves as an uncertainty measure. Here $w^i \in \bbR^{\abs{\cD_t^m}}$ denotes the vector of similarity scores between $x_i^{k_t}$ and all unaligned features in the other modality $\cD_t^{m}$, and $w^i_{(j)}$ denotes the $j$-th largest entry of $w^i$.   \label{line:compute_margin}
    \algcommentlight{Compute margin score as an uncertainty measure.}
    \STATE Select the subset $\crl{{x_i^{k_t}}}_{i=1}^B \subseteq \cC_t^{k_t}$ with the top-$B$ uncertainty scores (i.e., lowest margins), annotate them, and update $\cS_t \gets \cS_{t-1} \cup \crl{\prn{x_i^v, x_i^l}}_{i=1}^B$. 
    \label{line:data_selection}

    \STATE  \textbf{Model update.} Train multimodal model $\phi_t = \prn{\phi^v_t, \phi_t^l}$ on the updated annotation set $\cS_t$. \label{line:train_encoders}
\ENDFOR
    \ENSURE Actively trained multimodal model $\phi_T = \prn{\phi_T^v, \phi_T^l}$.
	\end{algorithmic}
\end{algorithm}

Our \cref{alg:multimodal_AL} integrates \emph{both diversity and uncertainty principles} into multimodal active learning with unaligned data. 
By restricting cross-modal uncertainty evaluation to the coreset constructed in Step 2, our algorithm
enables \emph{efficient data selection} with per-round runtime that scales \emph{linearly} in $\abs{\cD}$.
In contrast, a naive uncertainty-based selection over the entire dataset would require \emph{quadratic} time in $\abs{\cD}$,
which is computationally prohibitive.

We next explain the details of each of the three steps in \cref{alg:multimodal_AL}.

\paragraph{Step 1: Modality selection.}
This step selects the modality that is \emph{underrepresented} with respect to the current annotation $\cS_{t-1}$ (line \ref{line:modality_selection} in \cref{alg:multimodal_AL}).
To assess coverage, for each modality $k \in \crl{v, l}$,  
we compute the maximum distance of unaligned features $\cD_t^k \ldef \cD^k \setminus \cS_{t-1}^k$  
to their nearest neighbors in $\cS_{t-1}^k$:  
\begin{equation}
  d^k_t \ldef \max_{z_i \in \wb \phi_{t-1}(\cD_{t}^{k_t})} 
  \min_{z_j \in \wb \phi_{t-1}(\cS^{k_t}_{t-1}) } 
  \distmath (z_i, z_j),
  \nonumber
\end{equation}
where $\distmath$ is a distance metric,
and $\wb \phi(x)$ denotes normalized embeddings to ensure comparability across modalities.  
The modality with the largest distance value (i.e., least covered),  
$k_t \ldef \argmax_{k \in \crl{v, l}} d_t^k$,  
is selected for coreset construction in the next step.  
This step has a runtime upper bound of $O(\abs{\cD_t} \cdot \abs{\cS_{t-1}})$.

\paragraph{Step 2: Coreset construction.}
Given the selected modality $k_t$, in Step 2 (line \ref{line:coreset} of \cref{alg:multimodal_AL}), we construct a coreset $\cC_t^{k_t} \subseteq \cD_t^{k_t}$ of size $B_C \geq B$ that, when combined with already annotated data points $\cS_{t-1}^{k_t}$, maximally covers the unaligned data $\cD_t^{k_t}$. This is formalized in \cref{eq:coreset} of \cref{alg:multimodal_AL}.
Solving \cref{eq:coreset} exactly is NP-Hard \citep{cunninghamcombinatorial}.  
Following prior work in diversity-based active learning in the unimodal setting \citep{sener2018active},  
we employ a greedy approximation algorithm (\cref{alg:k-center})  
that guarantees a $2 \times \mathsf{OPT}$ solution.
\cref{alg:k-center} can be implemented in $O\left((B_C + \abs{\cS_{t-1}}) \cdot \abs{\cD_t}\right)$ time using a distance caching strategy: initially, we compute and cache the minimum distances between all candidate points in $\cD_t$ and the selection set $\cS_{t-1}$,  
incurring a runtime of $O(\abs{\cD_t} \cdot \abs{\cS_{t-1}})$.  
For subsequent iterations in the greedy selection process,  
we only need $O(\abs{\cD_t})$ operations to update the cache and select the next point.

\paragraph{Step 3: Uncertainty-based selection.} With the constructed coreset $\cC_t^{k_t}$, in Step 3 (lines \ref{line:step3}-\ref{line:data_selection} of \cref{alg:multimodal_AL}), we compute the margin score as the a measure of uncertainty and select the top-$B$ data points with the highest uncertainty (i.e., lowest margin scores) for multimodal annotation.\footnote{Although we adopt margin as the uncertainty measure in this step, our algorithm can be easily extended to alternative choices such as entropy or confidence. We use margin because it has been observed to consistently yield strong performance across active learning strategies \citep{zhang2024labelbench}.}
For each $x_i^{k_t} \in \cC_t^{k_t}$,  
we compute a vector of similarity scores $w^i \in \bbR^{\abs{\cD_t^m}}$  
between $x_i^{k_t}$ and all unaligned features in the other modality $m \ldef \crl{v, l} \setminus \crl{k_t}$.  
The calculation of similarity score is usually model-dependent; but a simple example is the inner product between representations (or its variants), e.g.,
$ w^i_j \ldef \ang{\phi_{t-1}^{k_t}(x_i^{k_t}), \phi_{t-1}^{m}(x_j^m)}$ for $x_j^m \in \cD_t^m$.
The uncertainty score is computed as the margin between the top two similarity scores: 
$u(x_i^{k_t}) \ldef w^i_{(1)} - w^i_{(2)}$.
We select the $B$ data points in $\cC_t^{k_t}$ with the \emph{lowest} margin scores.  
This step has a runtime of $O(B_C \cdot \abs{\cD_t})$.

 \begin{algorithm}[htbp]
	\caption{Greedy Approximation for Coreset Construction}
    \label{alg:k-center}
	\renewcommand{\algorithmicrequire}{\textbf{Input:}}	\renewcommand{\algorithmicensure}{\textbf{Output:}}
	\newcommand{\algorithmicbreak}{\textbf{break}}
    \newcommand{\BREAK}{\STATE \algorithmicbreak}
	\begin{algorithmic}[1]
\REQUIRE 
    Selected modality $k_t$, unaligned multimodal dataset $\cD_t^{k_t}$, selection size $B_C$, multimodal model $\phi_{t-1}$, current annotation set $\cS_{t-1}$.
    \STATE Initialize coreset set $\cC_t^{k_t} = \emptyset$. 
\WHILE{$\abs{\cC_t^{k_t}} < B_C$}

    \STATE  Select a data point $z_u$ such that:
\begin{equation*}
  z_u \ldef  \argmax_{z_i \in \phi_{t-1}(\cD_t^{k_t})} \, \min_{z_j \in \phi_{t-1}(\cS_{t-1}^{k_t})} \distmath(z_i, z_j).
\end{equation*}

    \STATE Update $\cC_t^{k_t} \gets \cC_t^{k_t} \cup \{z_u\}$.
\ENDWHILE
    \ENSURE Coreset $\cC_t^{k_t}$.
	\end{algorithmic}
\end{algorithm}

\paragraph{Computational complexity.} 
Summing the runtime of Steps 1–3,  
the per-round data acquisition complexity of \cref{alg:multimodal_AL} is upper bounded by  
$O\left((B_C + \abs{\cS_{t-1}}) \cdot \abs{\cD_t}\right)$,
where $\abs{\cS_{t-1}} = O(tB)$ and $\abs{\cD_t} \leq \abs{\cD}$.  
Across $T$ rounds,  
the total complexity is upper bounded by  
$O(T \cdot (B_C + TB) \cdot \abs{\cD})$,
as formalized in \cref{pro:multi_AL}.  

Assuming $T= O(1)$ and $B_C \ll \abs{\cD}$, both per-round and total complexity is linear in $\abs{\cD}$, which is typically large in real-world multimodal learning tasks \citep{gadre2024datacomp}.\footnote{We focus on the data acquisition complexity and its dependency on the data pool size $\abs{\cD}$, which is usually large for multimodal learning; note that the model training complexity is the same for all active learning algorithms and it's proportional to the training data size $\abs{\cS_{t-1}}$.}  
In contrast, a naive uncertainty-based approach would compute margin scores  
for \emph{all pairs} across modalities in the unaligned data,  
resulting in a per-round complexity of $O(\abs{\cD}^2)$,  
which is computationally expensive.  
Our algorithm avoids this quadratic bottleneck  
by computing cross-modal uncertainty \emph{only over the coreset (Step 3)}.

\begin{restatable}{proposition}{propMultiAL}
\label{pro:multi_AL}
The per-round data acquisition complexity of \cref{alg:multimodal_AL} is upper bounded by $O\left((B_C + \abs{\cS_{t-1}}) \cdot \abs{\cD_t}\right)$,  
resulting in an overall complexity of  
$O(T \cdot (B_C + TB) \cdot \abs{\cD})$.
\end{restatable}

\subsection{Extensions of \cref{alg:multimodal_AL}}
\label{sec:extensions}

\paragraph{Streaming-based multimodal active learning.} 
While \cref{alg:multimodal_AL} is originally designed for the pool-based setting,  
it can be easily adapted to streaming-based multimodal active learning.  
As introduced in \cref{sec:setting},  
in the streaming-based setting, the learner only has access to the current batch of stream data $\cD_t$.  
To adapt \cref{alg:multimodal_AL} to this setting,  
we simply replace line~\ref{line:unaligned_data} with the current batch of stream data $\cD_t$,  
and leave all other parts of the algorithm unchanged.  
As a result, aside from potential variation in the size of $\cD_t$,
the per-round computational complexity remains the same as in the pool-based case.

\paragraph{Active learning beyond vision-language models.}
Although our focus in this paper is on multimodal learning with vision-language data,  
\cref{alg:multimodal_AL} can be naturally extended to general multimodal settings  
with $m \geq 3$ unaligned modalities $\cD = \crl{\cD^{(1)}, \cD^{(2)}, \cdots, \cD^{(m)}}$. 
To support this generalization, we modify 
Step 1 and Step 3 of the algorithm while keeping 
Step 2 unchanged, since it operates solely within the selected modality. 
We discuss these changes below:
\begin{itemize}
[leftmargin=10pt, itemindent=*]
    \item 
For Step 1, we generalize the modality selection in \cref{eq:maxmin_dist} to select the least-covered modality \emph{among all $m$ modalities}, i.e., $k_t \ldef \argmax_{k \in \crl{1,2,...,m}} d_t^k$.
\item  For Step 3,  
given any data point $x^{k_t} \in \cD_t^{k_t}$ in the selected modality,  
let $u^j(x^{k_t})$ denote its cross-modal uncertainty score with respect to each unselected modality $j \in [m] \setminus \crl{k_t}$.  
We then define the \emph{overall uncertainty score} as  
$u(x^{k_t}) \ldef \sum_{j \in [m] \setminus \crl{k_t}} u^j(x^{k_t})$,
and use this score for uncertainty-based selection.  
\end{itemize}
Assuming $m = O(1)$,  
the computational complexity of this generalized algorithm remains order-wise the same  
as the analysis provided in \cref{pro:multi_AL}.

\section{Experiments}
\label{sec:exs}
We conduct extensive experiments to evaluate the effectiveness of our proposed algorithm. The experimental setup is described in \cref{sec:ex_setup}, followed by the main results and analyses in \cref{sec:main_results}. We defer additional implementation details and experimental results to \cref{app:experiments}.

\subsection{Experimental Setups}
\label{sec:ex_setup}
\paragraph{Datasets.} We conduct experiments on three multimodal datasets: ColorSwap \citep{burapacheep2024colorswap}, MS-COCO \citep{lin2014microsoft}, and DataComp \citep{gadre2024datacomp}. 
ColorSwap is designed to evaluate object-color matching with color-swapped image-caption pairs. MS-COCO is a large-scale image-caption dataset designed for object detection. DataComp consists of large-scale image-text pairs collected from Common Crawl. See \cref{tab:metrics} for detailed descriptions of these datasets.
We use ColorSwap for pool-based AL, and MS-COCO and DataComp for streaming-based AL.
For ColorSwap and MS-COCO, we initialize models from pretrained weights and study multimodal active learning in the \emph{finetuning} regime.  
For DataComp, we initialize models from random weights to evaluate multimodal active learning in the \emph{pretraining} regime.  

\begin{table}[t] \caption{Datasets used in our experiments and their corresponding evaluation metrics.} \label{tab:metrics}  \centering 

\begin{tabular}{c c  c } \toprule Datasets & Evaluation metrics & $\#$Data samples \\  \midrule ColorSwap &Scores: text, image, group \citep{burapacheep2024colorswap} & 1400\\ 
MS-COCO & R@1: I $\rightarrow$ T, T $\rightarrow$ I \citep{zhai2023sigmoid} & 118K\\ 
DataComp & Average score over 38 tasks \citep{gadre2024datacomp}  & 12.8M\\

\bottomrule \end{tabular} 
\end{table}

\paragraph{Baselines and models.}  
We compare our algorithm against three baselines: \Rand, \Coreset, and \Marg.  
\Rand serves as a passive learning baseline, randomly selecting data pairs for annotation.  
Since multimodal active learning with unaligned data is a new problem, no existing baselines directly apply.  
To better assess the performance of our \cref{alg:multimodal_AL}, we construct two additional baselines by adapting widely used unimodal AL methods---a diversity-based method (\Coreset, \citep{sener2018active}) and an uncertainty-based method (\Marg, \citep{settles2009active})---to our multimodal setting with unaligned data.\footnote{Full algorithmic details of \Coreset and \Marg are provided in \cref{app:exp_details_baselines}. We also report additional results for other adapted unimodal active learning algorithms in \cref{app:additional_baseline}.}  

We implement our algorithm and all baselines using the CLIP model \citep{radford2021learning} and its variants SigLIP \citep{zhai2023sigmoid} and LiT \citep{zhai2022lit}, evaluating across models of different sizes.  
Detailed hyperparameter settings are reported in \cref{sec:appendix_Hyperparameter_Settings}.

\paragraph{Evaluation metrics.}  
We adopt standard evaluation metrics for each dataset.  
For ColorSwap, we report the text score, image score, and group score, as proposed by \citet{burapacheep2024colorswap}.  
For MS-COCO, we report recall@1 for both image-to-text and text-to-image retrieval, as commonly used in the literature \citep{zhai2023sigmoid}.  
For DataComp, we follow \citet{gadre2024datacomp} and report the average score across 38 downstream tasks.  
All results are averaged over 4 random runs, with shaded regions in plots indicating ${2}/{3}$ of a standard deviation.  

We use ColorSwap dataset to study the pool-based setting, and MS-COCO and DataComp to study the streaming-based setting. The streaming buffer size $\abs{\cD_t}$ is set as 1024 for MS-COCO and 2048 for DataComp; 
additional experiments with different streaming buffer sizes are provided in \cref{app:exp_streaming_batch}.
The per-round data acquisition size $B$ is 70 for ColorSwap, 256 for MS-COCO, and 512 for DataComp.

\begin{figure}[t]
\centering
\includegraphics[width=\textwidth]{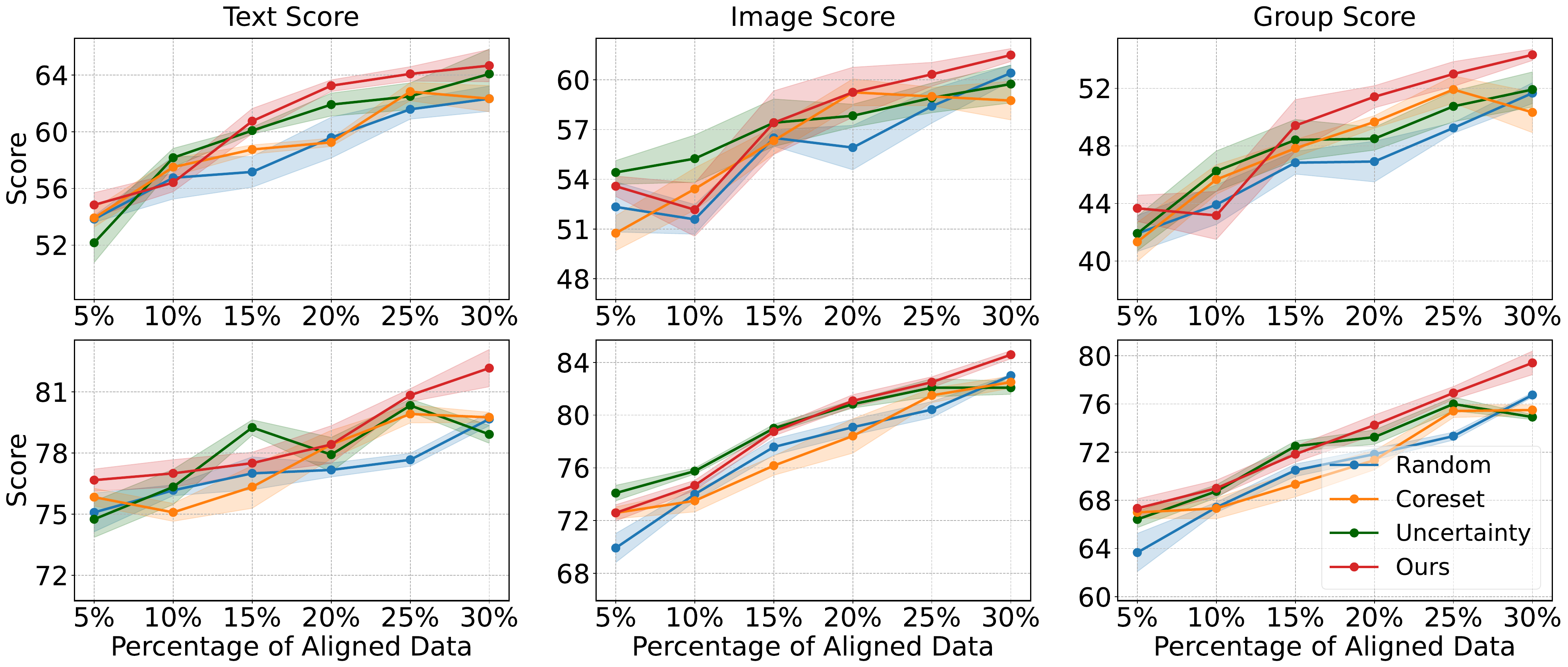}
\caption{Results of pool-based multimodal active learning on the ColorSwap dataset with CLIP-B32 (\emph{top}) and SigLIP-B16 (\emph{bottom}). We report text score (\emph{left}), image score (\emph{middle}), and group score (\emph{right}) as learning progresses.
}
\label{fig:color}
\end{figure}
\subsection{Main Results}
\label{sec:main_results}

\paragraph{Pool-based multimodal active learning.}
We compare \cref{alg:multimodal_AL} against three baselines on the ColorSwap dataset.  
As shown in \cref{fig:color}, our algorithm generally outperforms the baselines across all three metrics.  
Notably, with CLIP-B32, \cref{alg:multimodal_AL} achieves a group score of $49.42$ using only $15\%$ of the data, which is comparable to the group score reached by \Rand at $25\%$, corresponding to a $40\%$ reduction in annotation cost.  
Relative to \Marg, our algorithm achieves a group score of $51.42$ with $20\%$ of the data, while \Marg requires $30\%$ to reach a similar score—representing a $33\%$ reduction in cost.  
In addition to superior data efficiency, \cref{alg:multimodal_AL} is computationally more efficient than \Marg, which incurs higher complexity (\cref{sec:multimodal_AL}).

\begin{figure}[tbp]
\centering
\includegraphics[width=\textwidth]{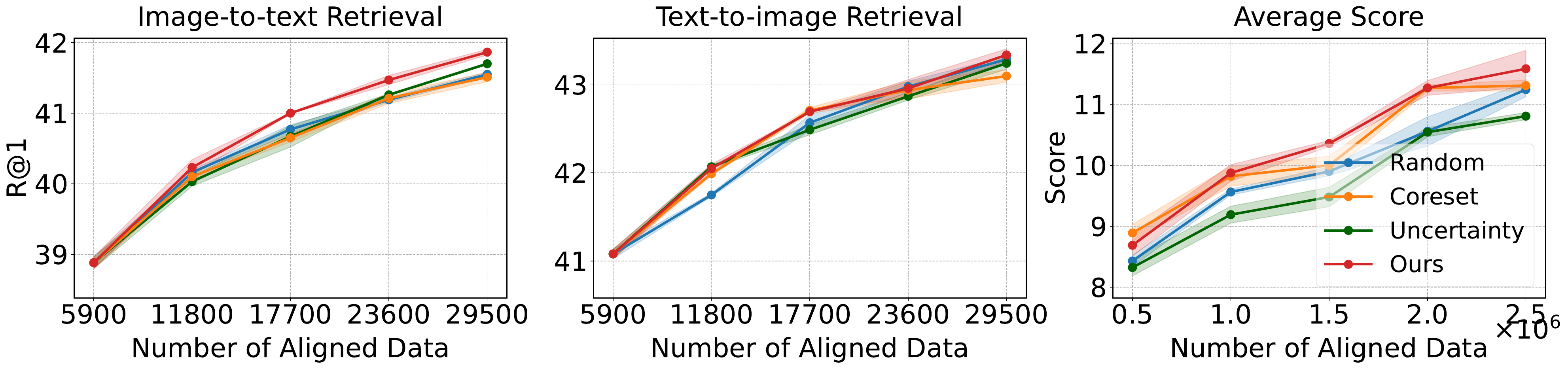}
\caption{Streaming-based multimodal active learning with the MS-COCO (\emph{left and middle}) and DataComp (\emph{right}) datasets using CLIP-B32.  
We report R@1 (image-to-text) (\emph{left}), R@1 (text-to-image) (\emph{middle}), and the average score  across 38 downstream tasks (\emph{right}).  
We report algorithm performance as learning progresses.
}

\label{fig:COCO_clip_}
\end{figure}
\begin{figure}[tbp]
\centering
\includegraphics[width=\textwidth]{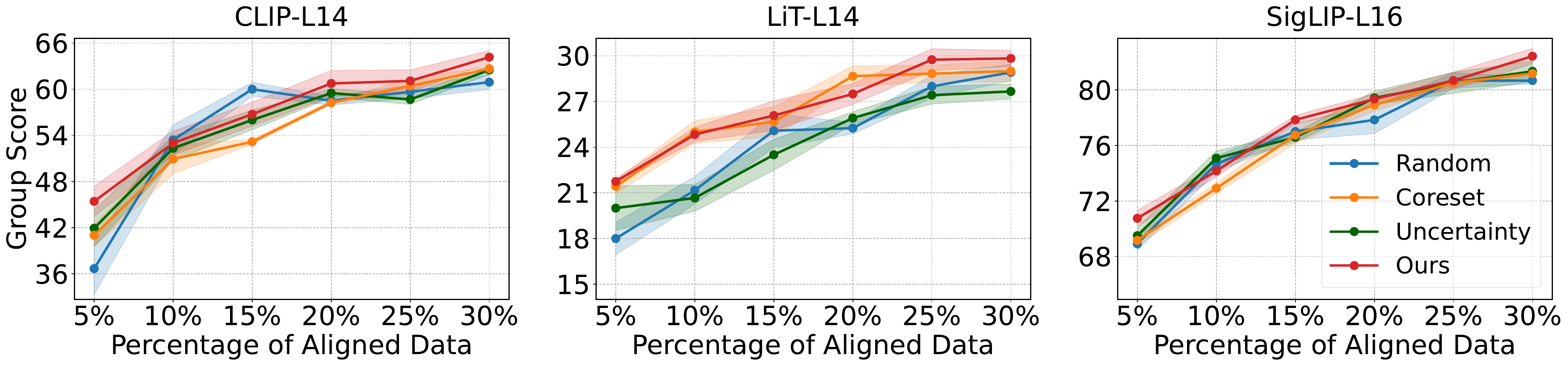}
\caption{Group scores on the ColorSwap dataset in the pool-based setting, using CLIP-L14 (\emph{left}), LiT-L14 (\emph{middle}), and SigLIP-L16 (\emph{right}).}  
\label{fig:color_n}
\end{figure}

\paragraph{Streaming-based multimodal active learning.} 
We next evaluate \cref{alg:multimodal_AL} in the streaming-based setting on MS-COCO and DataComp datasets.  
On MS-COCO (left and middle panels of \cref{fig:COCO_clip_}), our algorithm consistently outperforms all baselines across both retrieval metrics.  
For example, on R@1 (image-to-text), \cref{alg:multimodal_AL} achieves a score of $41.47$ using $23{,}600$ samples ($20\%$ of the data), matching the performance of \Coreset and \Rand with $25\%$ of the data.  
This corresponds to a $20\%$ reduction in annotation cost.  

To examine performance at larger scales, we further conduct experiments on the DataComp dataset (right panel of \cref{fig:COCO_clip_}), reporting the average score across 38 downstream tasks. 
Our algorithm outperforms all baselines, with particularly large margins over \Rand and \Marg.  
Relative to \Coreset, \cref{alg:multimodal_AL} achieves clear improvements as the number of aligned pairs grows, reaching a score of $11.59$ with $2.5$M pairs.
For context, training on the full $12.8$M aligned pairs (the performance skyline) yields a score of $13.20$ \citep{gadre2024datacomp}.
Thus, our method attains $87.80\%$ of the skyline with just $2.5$M pairs, whereas \Coreset reaches only $85.68\%$ of the skyline (with a score of 11.31).

\paragraph{Robustness across CLIP variants and model sizes.}  
To evaluate the robustness of \cref{alg:multimodal_AL} across architectures and model sizes, we conduct additional experiments with multiple CLIP variants, including SigLIP \citep{zhai2023sigmoid} and LiT \citep{zhai2022lit}, as well as models of different sizes.  
Given the high runtime cost of the streaming-based setting, most experiments are performed in the pool-based scenario.  
\cref{fig:color_n} reports group scores for large-scale variants, where \cref{alg:multimodal_AL} consistently outperforms the baselines.  
Results for text and image scores are deferred to \cref{app:additional_exp}, and show similar trends.

\section{Analyses and Ablations}
\label{sec:discussion}
\begin{table}[tbp]
\centering
\caption{Ablation study of \cref{alg:multimodal_AL} with different modality selection strategies in the pool-based setting using CLIP-B32 and SigLIP-B16.  
We report group scores with $30\%$ of aligned data.}
\label{tab:ablation_1}
\begin{tabular}{lcccc}
\toprule
\textbf{Model} & \textbf{Random} & \textbf{Text-only} & \textbf{Image-only}  & \textbf{\Ours} \\
\midrule
CLIP-B32 & 52.4\scriptsize$\pm$3.7 & 52.8\scriptsize$\pm$1.4 & 50.7\scriptsize$\pm$2.3 & \textbf{54.3\scriptsize$\pm$1.2} \\
SigLIP-B16  & 75.50\scriptsize$\pm$2.1  & 77.67\scriptsize$\pm$2.9 & 76.89\scriptsize$\pm$1.9 & \textbf{79.42\scriptsize$\pm$2.9} \\
\bottomrule
\end{tabular}
\end{table}

\begin{table}[tbp]
\centering
\caption{Case study of \cref{alg:multimodal_AL}, recording margin scores at Step~3 under the pool-based setting with CLIP-B32.  
We report average margin scores (scaled by $\times 10^{-2}$) for both correctly and incorrectly matched groups (with respect to the ground-truth alignments) across different percentages of aligned data pairs.}  
\begin{tabular}{lcccccc}
\toprule
\textbf{Percentage of Aligned Data} & \textbf{5\%} & \textbf{10\%} & \textbf{15\%} & \textbf{20\%} & \textbf{25\%} & \textbf{30\%} \\
\midrule
Correctly matched group  & 3.8\scriptsize$\pm$0.2 &4.0\scriptsize$\pm$0.1  & 4.2\scriptsize$\pm$0.4 & 4.7\scriptsize$\pm$0.2 & 4.8\scriptsize$\pm$0.2 & 5.2\scriptsize$\pm$0.2 \\
Incorrectly matched group&0.9\scriptsize$\pm$0.1  &  1.0\scriptsize$\pm$0.2&1.3\scriptsize$\pm$0.1  &1.4\scriptsize$\pm$0.2 &1.4\scriptsize$\pm$0.1  &1.5\scriptsize$\pm$0.5  \\
\bottomrule
\end{tabular}
\label{tab:margin_scores_in_appendix}
\end{table}

\paragraph{Effectiveness of modality selection.}  
\begin{figure}[t]
\centering
\includegraphics[height=3.8cm]{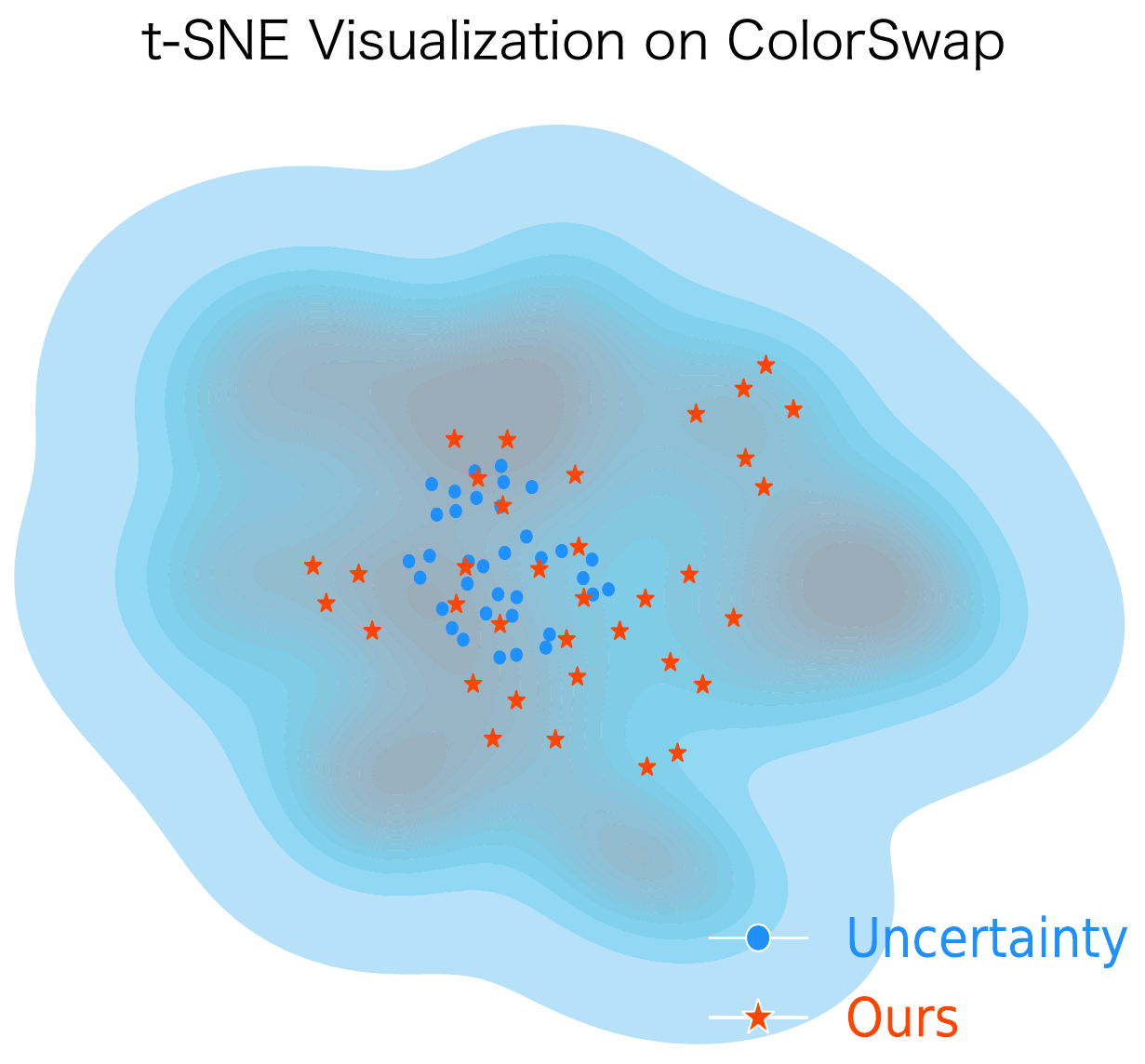}\hfill
\includegraphics[height=4.5cm]{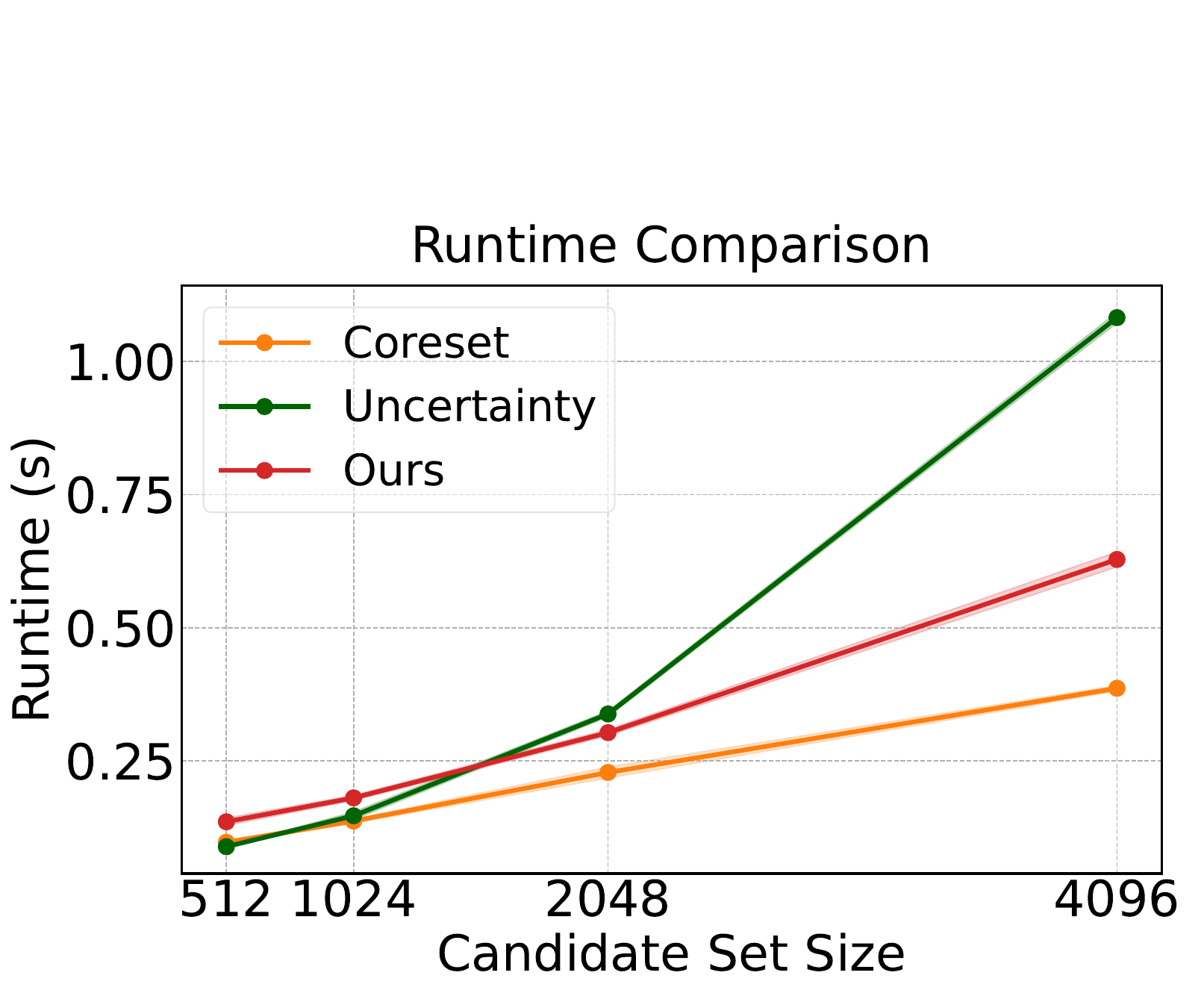}\hfill
\includegraphics[height=3.56cm]{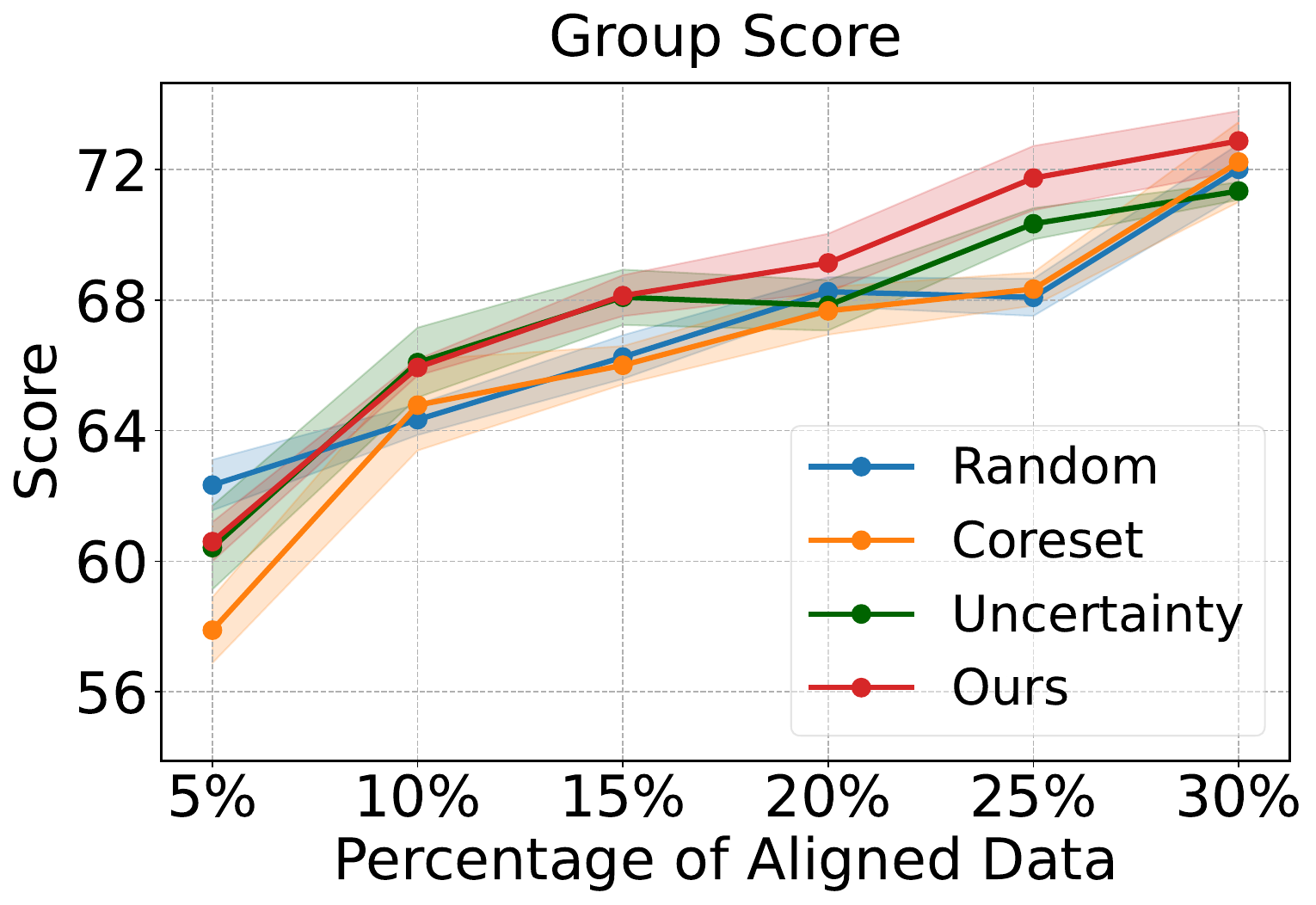}\hfill
\caption{\emph{Left:} t-SNE visualization of image-modality embeddings from the ColorSwap dataset, comparing our method (\cref{alg:multimodal_AL}) with an uncertainty-based baseline (\Marg).  
  Points selected by the baseline are shown as blue circles, while points selected by our method are shown as red stars.  
  Blue density contours represent the distribution of all data. 
   \emph{Middle:} Per-round data acquisition time (seconds) across candidate set sizes for different acquisition strategies using CLIP-B32. 
 \emph{Right:} Robustness to missing data and alignments on ColorSwap (SigLIP-B16): we remove 50\% of images while keeping all captions, creating many unmatched captions. All methods query labels from the image modality. }
\label{fig:analysis_1}
\end{figure}

We first assess the effectiveness of the modality selection strategy in \cref{alg:multimodal_AL} (Step 1).  
Following the same experimental setup as in \cref{fig:color}, we compare our approach against three alternatives: randomly choosing a modality, always selecting the text modality, and always selecting the image modality.  
As reported in \cref{tab:ablation_1}, our proposed strategy achieves the best performance across both CLIP and SigLIP models.  
For example, it attains a group score of $54.3$ with the CLIP model, representing up to a $7\%$ improvement over the image-only strategy, with all other components held fixed.

\paragraph{Effectiveness of margin score in data selection.}

To evaluate the effectiveness of the margin score in \cref{alg:multimodal_AL} (Step 3), we analyze its behavior across different percentages of aligned data pairs. At each iteration, Step~3 computes pseudo-alignments for all unaligned data points by selecting their most likely match based on similarity.  
We then partition the data into correctly matched and incorrectly matched groups (with respect to ground-truth alignments) and report their average margin scores in \cref{tab:margin_scores_in_appendix}.  
As expected, incorrectly matched samples consistently exhibit lower margin scores than correctly matched ones, reflecting higher uncertainty and thus greater value for active selection.  
Since these pseudo-alignments are computed on unaligned data not yet used in training, the results confirm that the margin score provides a robust and meaningful uncertainty signal even in noisy, unaligned settings.
\looseness=-1

\paragraph{Synergy of uncertainty and diversity in data selection.}

As shown in \cref{sec:main_results}, \cref{alg:multimodal_AL} consistently outperforms all baselines.  
We attribute this advantage to its ability to prioritize data points that are both \emph{uncertain} (Step~3) and \emph{diverse} (Step~2).  
To test this hypothesis, we visualize the ColorSwap dataset using t-SNE on image-modality embeddings, comparing data selected by \cref{alg:multimodal_AL} with those chosen by \Marg.  
As shown in \cref{fig:analysis_1} (left), data points selected by \Marg (blue circles) tend to concentrate within a narrow uncertain region.
In extreme cases---for example, when multiple highly similar or duplicated data points are all uncertain---pure uncertainty-based selection may repeatedly select redundant samples, providing little additional learning benefit. 
In contrast, data points selected by \cref{alg:multimodal_AL} (red stars) not only concentrate near high-uncertainty regions but also spread more broadly across the embedding space, reflecting greater diversity.  
This joint emphasis on uncertainty and diversity provides a key advantage of \cref{alg:multimodal_AL}.  
\looseness=-1

\paragraph{Empirical data acquisition runtime comparison.}
\cref{fig:analysis_1} (middle) reports the per-round data acquisition time of each algorithm using CLIP-B32. 
 As expected, \cref{alg:multimodal_AL} is slower than the \Coreset baseline because it explicitly includes a coreset construction step. However, as the candidate set size increases, our algorithm becomes substantially more efficient than the \Marg baseline. 
 These runtime results are consistent with our theoretical analysis in \cref{sec:multimodal_AL} and confirm that our algorithm is more efficient than the Uncertainty baseline.

\paragraph{Evaluation under different metrics.}
To further assess alignment quality beyond the metrics used in ColorSwap, we additionally report pseudo-labeling accuracy and the GroupMatch score \citep{zhu2026test}. \Cref{tab:pseudo_label_groupmatch} summarizes results with CLIP-B32, where our method continues to outperform all baselines on both pseudo-labeling accuracy and GroupMatch, further supporting the effectiveness of the proposed approach.

\begin{table}[tbp]
\centering
\caption{Pseudo-labeling accuracy and GroupMatch score on ColorSwap using CLIP-B32.}
\begin{tabular}{lcccc}
\toprule
Metric & Random & Coreset & Uncertainty & Ours \\
\midrule
Pseudo-labeling accuracy & 79.81\scriptsize$\pm$0.73& 78.52\scriptsize$\pm$0.86 & 79.67\scriptsize$\pm$0.69 &  \textbf{80.39\scriptsize$\pm$0.71}  \\
GroupMatch score & 86.57\scriptsize$\pm$0.29  & 85.57\scriptsize$\pm$0.15 & 86.43\scriptsize$\pm$0.38 &  \textbf{87.14\scriptsize$\pm$0.22}\\
\bottomrule
\end{tabular}
\label{tab:pseudo_label_groupmatch}
\end{table}

\paragraph{Robustness to coreset hyperparameter $B_C$.}  
To examine robustness to the coreset hyperparameter $B_C$ (Step 2), we conduct experiments across model families (CLIP and SigLIP) and scales (CLIP-B32 and SigLIP-L16).  
As shown in \cref{fig:ablation_2}, performance remains stable across different values of $B_C$ and annotation costs, demonstrating that our method is robust to this parameter.  
In practice, we select the value of $B_C$ from the set $\crl{1.5B, 2B, 2.5B}$, depending on the dataset and model. We provide detailed hyperparameter selections 
in \cref{sec:appendix_Hyperparameter_Settings}.

\begin{figure}[tbp]
  \centering
  \includegraphics[width=0.7\textwidth]{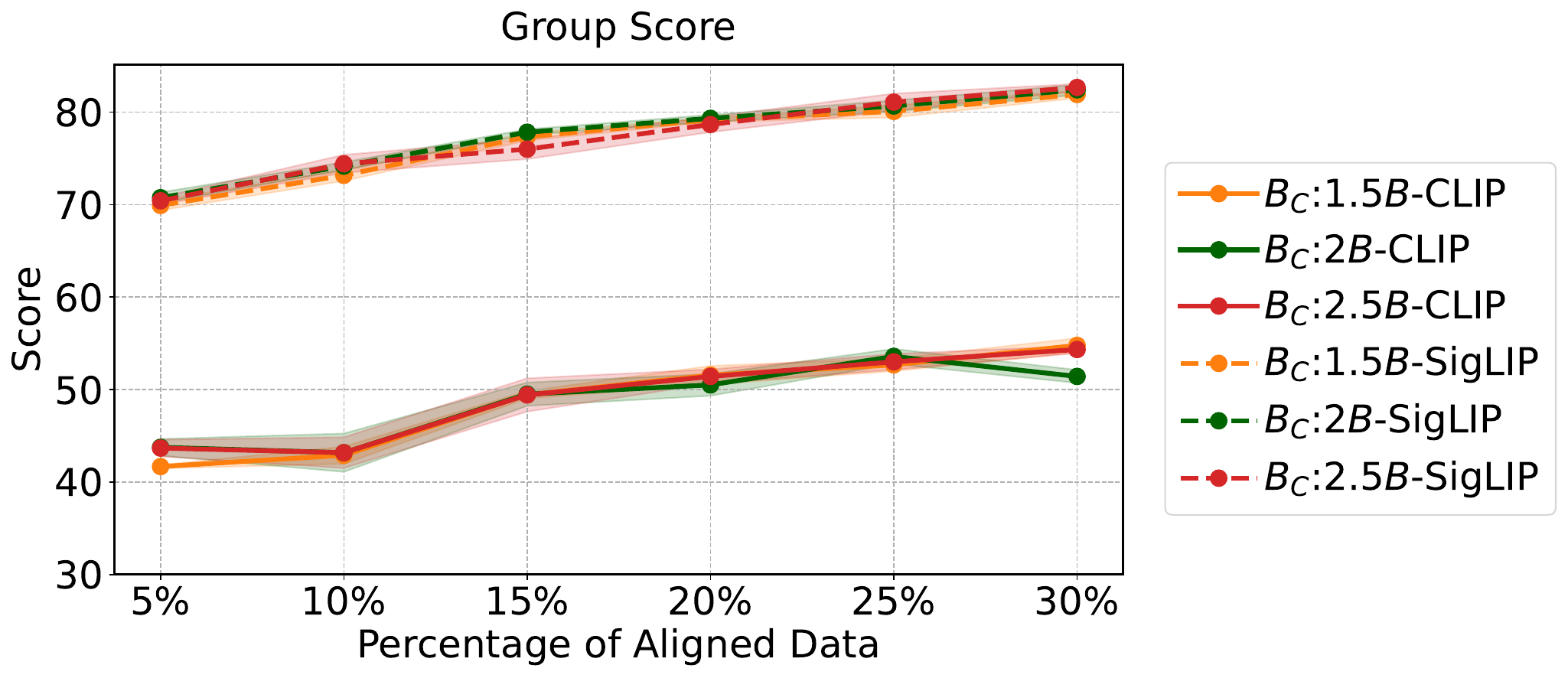} 
  \caption{Parameter study of \cref{alg:multimodal_AL} with different values of $B_C$ in the pool-based setting using CLIP-B32 and SigLIP-L16.  
  We report group scores as learning progresses.}
  \label{fig:ablation_2}
\end{figure}
\paragraph{Robustness to missing data and alignments.}
While our primary experiments focus on settings where valid cross-modal alignments exist for all instances---allowing us to isolate the core challenges of multimodal active learning with unaligned data (\cref{sec:unique})---real-world datasets often contain missing data and alignments.
To evaluate robustness under such conditions, we conduct an experiment on a deliberately perturbed version of the ColorSwap dataset. Specifically, we retain all captions but remove half of the images, creating a setting where many captions no longer have valid visual counterparts. 
We further restrict annotation to proceed only from the image modality (for our method and all baselines), since certain captions cannot be aligned to any image. The results, reported in \cref{fig:analysis_1} (right), show that our method continues to achieve the best performance under this noisy setting, suggesting robustness to missing data and alignments.

\section{Conclusion}
\label{sec:conclusion}

We presented the first study of active learning in multimodal settings \emph{with unaligned data}, addressing the key challenges of bidirectional alignment and annotation across large cross-modal candidate spaces.
By integrating uncertainty- and diversity-based selection in a modality-aware design, we developed an efficient algorithm applicable to both pool-based and streaming-based scenarios.
Experiments on benchmark datasets show that our approach reduces annotation requirements by up to $40\%$ while maintaining model performance, highlighting the promise of active learning for scalable and cost-effective multimodal learning.

Our current method is specifically designed for multimodal representation learning, where supervision is provided through cross-modal alignments and training is typically contrastive. Extending our method to other multimodal paradigms---such as multimodal generative learning---would likely require non-trivial modifications to the acquisition strategy and uncertainty estimation. We view this as an important and promising direction for future work.

\bibliography{refs}

\newpage
\appendix
\section{Supporting Results from \cref{sec:methods}}
\label{app:analysis}

\propMultiAL*

\begin{proof}
We analyze the data acquisition complexity of \cref{alg:multimodal_AL} as follows.
\begin{itemize}
    \item \textit{Per-round complexity.}  
    We analyze the runtime of each major step in the algorithm:

    \begin{itemize}
        \item \textit{Line 5.}  
        The main computational cost arises from evaluating \cref{eq:maxmin_dist},  
        which involves iterating over both $\cD_t$ and $\cS_{t-1}$.  
        The resulting runtime is upper bounded by $O(\abs{\cD_t} \cdot \abs{\cS_{t-1}})$.

        \item \textit{Line 6.}  
        For coreset construction, we use the greedy approximation algorithm in \cref{alg:k-center}, which can be implemented in $O\left((B_C + \abs{\cS_{t-1}}) \cdot \abs{\cD_t}\right)$ time using a distance caching strategy: initially, we compute and cache the minimum distances between all candidate points in $\cD_t$ and the selection set $\cS_{t-1}$,  
incurring a runtime of $O(\abs{\cD_t} \cdot \abs{\cS_{t-1}})$.  
For subsequent iterations in the greedy selection process,  
we only need $O(\abs{\cD_t})$ operations to update the cache and select the next point.
\looseness=-1

        \item \textit{Lines 7--10.}  
        These steps perform uncertainty-based selection.  
        Computing uncertainty scores over $B_C$ coreset candidates against the other modality  
        requires $O(B_C \cdot \abs{\cD_t})$ operations.
    \end{itemize}

    Summing the contributions from each step,  
    the per-round data acquisition complexity is upper bounded by  
    \[
    O\left((B_C + \abs{\cS_{t-1}}) \cdot \abs{\cD_t}\right).
    \]

    \item \textit{Overall complexity.}  
    Since $\abs{\cS_{t-1}} = O(tB)$ and $\abs{\cD_t} \leq \abs{\cD}$,  
    the total complexity over $T$ rounds is
    \[
    O\left(\sum_{t=1}^{T} (B_C + \abs{\cS_{t-1}}) \cdot \abs{\cD_t}\right) 
    = O(T \cdot (B_C + TB) \cdot \abs{\cD}).
    \]
\end{itemize}
\end{proof}

\section{Other Details for Experiments}
\label{app:experiments}

\subsection{Additional Details and Baselines}
\label{app:exp_details_baselines}

\paragraph{Additional details on datasets.}  
We use the 2017 release of MS-COCO, which contains approximately 118K training images (train split) and 5K validation images (val split).\footnote{\url{https://cocodataset.org}}  
Since only these two splits provide captions, we use the train split for training and the val split for testing.  
Each image in MS-COCO is paired with five captions; in our experiments, we use the first caption for each image.

\paragraph{Baselines.}  
\label{sec:append_baselines}  
We provide full implementations of the two baseline methods used in our experiments:  
\Coreset (\cref{alg:baseline_coreset}) and \Marg (\cref{alg:baseline_uncertainty}).  
Both are adapted from diversity-based and uncertainty-based active learning algorithms originally developed for unimodal settings.  

In the multimodal setting, \cref{alg:baseline_coreset} randomly selects a modality and then constructs a coreset within that modality using a greedy algorithm.  
\cref{alg:baseline_uncertainty} computes margin-based uncertainty scores in both directions (text $\rightarrow$ image and image $\rightarrow$ text), and selects the top-$B$ most uncertain instances for multimodal annotation.  

To extend these baselines to the streaming setting, we adopt the same strategy as in \cref{alg:multimodal_AL}:  
line~\ref{line:current_stream} in both \cref{alg:baseline_coreset} and \cref{alg:baseline_uncertainty} is replaced with the current batch of stream data $\cD_t$, while the remainder of each algorithm is left unchanged.  
We use Euclidean distance as the default metric in \cref{alg:multimodal_AL}, \cref{alg:baseline_coreset}, and \cref{alg:baseline_uncertainty}.
Additional results using alternative distance metrics in our algorithm are provided in \cref{app:metric}.

\begin{algorithm}[htbp]
	\caption{Multimodal Coreset Selection}
    \label{alg:baseline_coreset}
	\renewcommand{\algorithmicrequire}{\textbf{Input:}}	\renewcommand{\algorithmicensure}{\textbf{Output:}}
	\newcommand{\algorithmicbreak}{\textbf{break}}
    \newcommand{\BREAK}{\STATE \algorithmicbreak}
	\begin{algorithmic}[1]
\REQUIRE 
Unaligned multimodal dataset $\cD=\{\cD^v,\cD^l\},$ number of iterations $T,$ per-round selection size $B$.
\STATE Initialize multimodal model $\phi_0=\{\phi^v_0,\phi^l_0\}$ with random or pretrained weights.
\STATE Initialize the annotation set $\cS_0=\emptyset.$

\FOR{$t= 1, \cdots, T$}
\STATE Consider unaligned data pool $\cD_t\ldef  \cD \setminus \cS_{t-1}$.
\label{line:current_stream}
\STATE Randomly select a modality $k_t \in \{v,l\}$.
\FOR{$c=1,\cdots,B$}
\STATE   $z_u =  \argmax_{z_i \in \phi(\cD_t^{k_t})} \, \min_{z_j \in \phi(\cS_{t-1}^{k_t})} \distmath(z_i, z_j)$.\label{line:base_coreset_data_selectionway}

\STATE Annotate $z_u$ and add $\prn{z_i^v, z_i^l}$ to $\cS_t$.
\ENDFOR
\STATE Train multimodal model $\phi_t=(\phi^v_t,\phi^l_t)$ on the updated annotation set $\cS_t.$
\ENDFOR
    \ENSURE Actively trained multimodal model $\phi_T=(\phi^v_T,\phi^l_T)$.
	\end{algorithmic}
\end{algorithm}
\begin{algorithm}[htbp]
	\caption{Multimodal Uncertainty-based Data Selection}
    \label{alg:baseline_uncertainty}
	\renewcommand{\algorithmicrequire}{\textbf{Input:}}	\renewcommand{\algorithmicensure}{\textbf{Output:}}
	\newcommand{\algorithmicbreak}{\textbf{break}}
    \newcommand{\BREAK}{\STATE \algorithmicbreak}
	\begin{algorithmic}[1]
\REQUIRE Unaligned multimodal dataset $\cD=\{\cD^v,\cD^l\}$, number of iterations $T,$ per-round selection size $B.$
\STATE Initialize multimodal model $\phi_0=\{\phi^v_0,\phi^l_0\}$ with random or pretrained weights.
\STATE Initialize the annotation set $\cS_0=\emptyset.$
\FOR{$t=1,\cdots,T$}
\STATE Consider unaligned data pool $\cD_t\ldef  \cD \setminus \cS_{t-1}$

    \STATE For each modality $k \in \crl{v, l}$ and for each data $x_i^{k} \in \cC_t^{k}$, compute its margin score $u(x_i^{k}) \ldef w^{i}_{(1)} - w^i_{(2)}$, which serves as an uncertainty measure. Here $w^i \in \bbR^{\abs{\cD_t^m}}$ denotes the vector of similarity scores between $x_i^{k_t}$ and all unaligned features in the other modality $m \ldef \crl{v, l} \setminus \crl{k}$, and $w^i_{(j)}$ denotes the $j$-th largest entry of $w^i$.  
\algcommentlight{Calculate the margin scores from both directions.}

    \STATE Select $B$ data points with smallest margin scores with respect to $\crl{x^v_i}_{i=1}^{\abs{\cD^{v}_t}} \cup \crl{x^{l}_j}_{j=1}^{\abs{\cD^{l}_t}}$, annotate them, and update $\cS_t \gets \cS_{t-1} \cup \crl{\prn{x_i^v, x_i^l}}_{i=1}^B$.
    \algcommentlight{If, within the top-$B$ selection, a selected image is paired with another selected text, then continue the selection process until getting $B$ aligned pairs.}

\STATE Train multimodal model $\phi_t=(\phi^v_t,\phi^l_t)$ on the updated annotation set $\cS_t.$
\ENDFOR
    \ENSURE Actively trained multimodal model $\phi_T=(\phi^v_T,\phi^l_T)$.

	\end{algorithmic}
\end{algorithm}

\subsection{Hyperparameter Settings}
\label{sec:appendix_Hyperparameter_Settings}

\cref{tab:hypers_color_clip,tab:hypers_color,tab:hypers_datacomp} list the hyperparameters used in our experiments across datasets and model variants.  

\begin{table*}[hbtp]
\caption{Hyperparameter settings for the ColorSwap dataset using CLIP model.}
\label{tab:hypers_color_clip}
\centering
\begin{tabular}{l c c }
\toprule
Hyperparameters & CLIP-B32  & CLIP-L14 \\
\midrule
Epochs & 50 & 50  \\
Batch Size & 70 & 35 \\
  Optimizer &  AdamW  & AdamW  \\
Weight Decay &0.1& 0.1  \\
Learning Rate & $2 \times 10^{-5}$ & $1 \times 10^{-5}$  \\

$B_C$ & $2.5B$ & $2.5B$  \\
\bottomrule
\end{tabular}
\end{table*}

\begin{table*}[hbtp]
\caption{Hyperparameter settings for the ColorSwap dataset using SigLIP and LiT.}
\label{tab:hypers_color}
\centering
\begin{tabular}{l c c c}
\toprule
Hyperparameters & SigLIP-B16  & SigLIP-L16 &LiT-L14 \\
\midrule
Epochs & 80 & 80& 50 \\
Batch Size & 70 & 35&35 \\
Optimizer & AdamW &  AdamW   &  AdamW\\
Weight Decay &0.1& 0.1 & 0.1 \\
Learning Rate & $2 \times 10^{-5}$ & $1 \times 10^{-5}$& $1 \times 10^{-5}$ \\

$B_C$ & $1.5B$ & $2.0B$ & $2.5B$\\
\bottomrule
\end{tabular}
\end{table*}

\begin{table*}[hbtp]
\caption{Hyperparameter settings for the MS-COCO and DataComp datasets using CLIP model.}
\label{tab:hypers_datacomp}
\centering
\begin{tabular}{l c c }
\toprule
Hyperparameters & MS-COCO  & DataComp \\
\midrule
Epochs & 10 & 1  \\
Batch Size & 256 & 512 \\
Optimizer & AdamW &  AdamW   \\
Weight Decay &$1 \times 10^{-4}$&  0.2 \\
Learning Rate & $1 \times 10^{-5}$ & $6.25 \times 10^{-5}$   \\
$B_C$ & $2B$ & $2.5B$  \\
\bottomrule
\end{tabular}
\end{table*}

\section{Additional Experimental Results}  

  \subsection{Additional Results on Different CLIP Variants and Model Sizes}
\label{app:additional_exp}
We present additional experimental results in \cref{fig:color_L14,fig:color_sigmoid_large,fig:color_LiT_txt_Large}.  
The conclusions from the main results in \cref{sec:main_results} remain consistent across different CLIP variants and model sizes:  
\cref{alg:multimodal_AL} continues to outperform all baselines.

\begin{figure}[htbp]
\centering
\includegraphics[width=\textwidth]{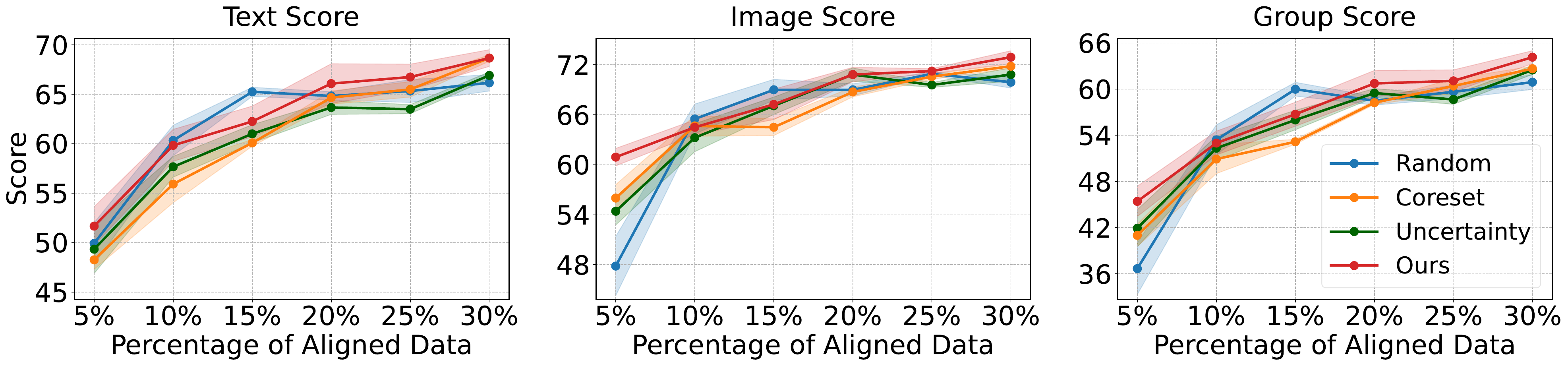}
\caption{Results of pool-based multimodal active learning on the ColorSwap dataset with CLIP-L14. We report text score (\emph{left}), image score (\emph{middle}), and group score (\emph{right}) as learning progresses.
}
\label{fig:color_L14}
\end{figure}

\begin{figure}[htbp]
\centering
\includegraphics[width=\textwidth]{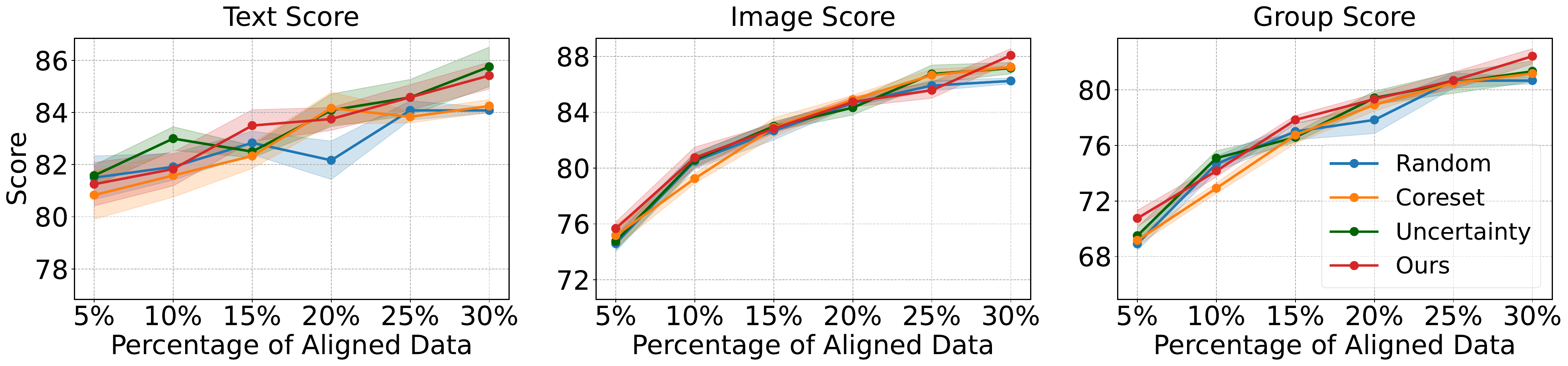}
\caption{Results of pool-based multimodal active learning on the ColorSwap dataset with SigLIP-L16. We report text score (\emph{left}), image score (\emph{middle}), and group score (\emph{right}) as learning progresses.
}

\label{fig:color_sigmoid_large}
\end{figure}

\begin{figure}[htbp]
\centering
\includegraphics[width=\textwidth]{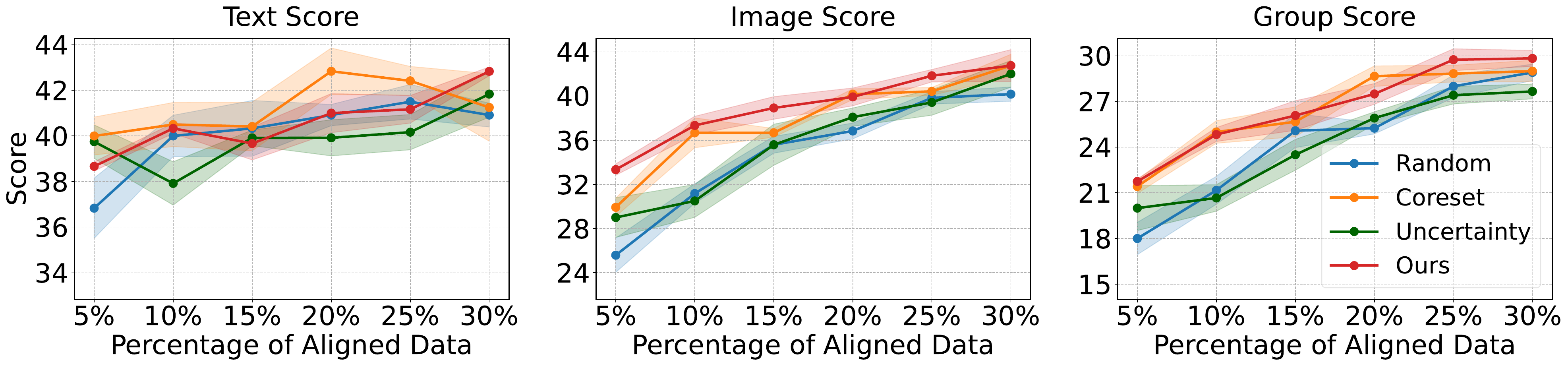}
\caption{Results of pool-based multimodal active learning on the ColorSwap dataset with LiT-L14. We report text score (\emph{left}), image score (\emph{middle}), and group score (\emph{right}) as learning progresses.
}

\label{fig:color_LiT_txt_Large}
\end{figure}

  \subsection{Streaming Active Learning with Different Streaming Buffer Sizes}
  \label{app:exp_streaming_batch}

  We conduct experiments on the DataComp dataset under two streaming buffer sizes, 1024 and 2048. 
  In both cases, we select 25\% of the data in the buffer, resulting in training batch sizes of 256 and 512, respectively.
  As shown in \cref{fig:datacomp_diff_size} (left and middle), our method achieves the best performance in both cases, indicating robustness to the choice of streaming buffer size.

  \subsection{Learning with Different Distance Metrics}
  \label{app:metric}
In our main experiments, we use the Euclidean distance as the default metric. To examine the effect of alternative distance functions, we additionally conducted experiments using cosine distance, defined as $d_{\mathsf{cos}} (x, y) = 1 - \mathsf{CosineSimilarity} (\frac{x}{||x||}, \frac{y}{||y||})$. 
As shown in \cref{tab:comparison_cosine}, our method exhibits comparable performance under cosine distance, achieving similar trends and performance levels to those obtained with Euclidean distance. This suggests that our method is robust to the choice of distance metric.
\begin{table}[htbp]
\centering
\caption{Results of pool-based multimodal active learning under different distance metrics on the ColorSwap dataset using SigLIP-B16.}
\label{tab:comparison_cosine}
\begin{tabular}{lcccccc}
\toprule
\textbf{Distance Metric} & \textbf{5\%} & \textbf{10\%} & \textbf{15\%} & \textbf{20\%} & \textbf{25\%} & \textbf{30\%} \\
\midrule
Euclidean Distance     & 67.34\scriptsize$\pm$2.36  & 69.00\scriptsize$\pm$2.02 & 71.83\scriptsize$\pm$1.97 & 74.25\scriptsize$\pm$2.50 & 76.92\scriptsize$\pm$1.59 & 79.42\scriptsize$\pm$2.94 \\
Cosine Distance & 65.17\scriptsize$\pm$2.67 & 68.17\scriptsize$\pm$2.14 & 74.17\scriptsize$\pm$2.09 & 74.67\scriptsize$\pm$1.96 & 77.83\scriptsize$\pm$1.73 & 78.67\scriptsize$\pm$2.27 \\
\bottomrule
\end{tabular}
\end{table}

  \subsection{Comparison against Additional Adapted Uniform AL Baselines}
  \label{app:additional_baseline}
Beyond adapting the unimodal \Coreset and \Marg baselines to our setting, we additionally evaluate an adapted version of the unimodal \badge algorithm \citep{Ash2020Deep}. 
We adapt \badge in the most direct manner: we first randomly select a modality and then apply \badge as if operating in a unimodal setting with classification losses.
Results in \cref{fig:datacomp_diff_size} (right) indicate that our method consistently outperforms this adapted BADGE baseline.

Although both \badge and our approach combine uncertainty and diversity principles, we hypothesize that \badge’s weaker performance stems from a fundamental mismatch with our problem structure. Specifically, \badge performs clustering in a gradient embedding space of dimension $O(Kd)$, where $d$ is the embedding dimension and $K$ is the number of classes. While $K$ is typically small in unimodal classification tasks, in multimodal learning with bidirectional alignment, $K$ effectively scales with the number of data points, since texts and images are generally unique. This issue is highlighted in \cref{sec:unique} as a core challenge of our setting. The resulting extremely high-dimensional gradient embeddings make \badge’s clustering step ineffective in practice.
In addition, \badge incurs quadratic per-round computational complexity in the number of data points, whereas our method is explicitly designed to achieve linear per-round complexity, which is critical for scalability in large multimodal datasets.

\begin{figure}[t]
\centering
\includegraphics[width=.65\linewidth]{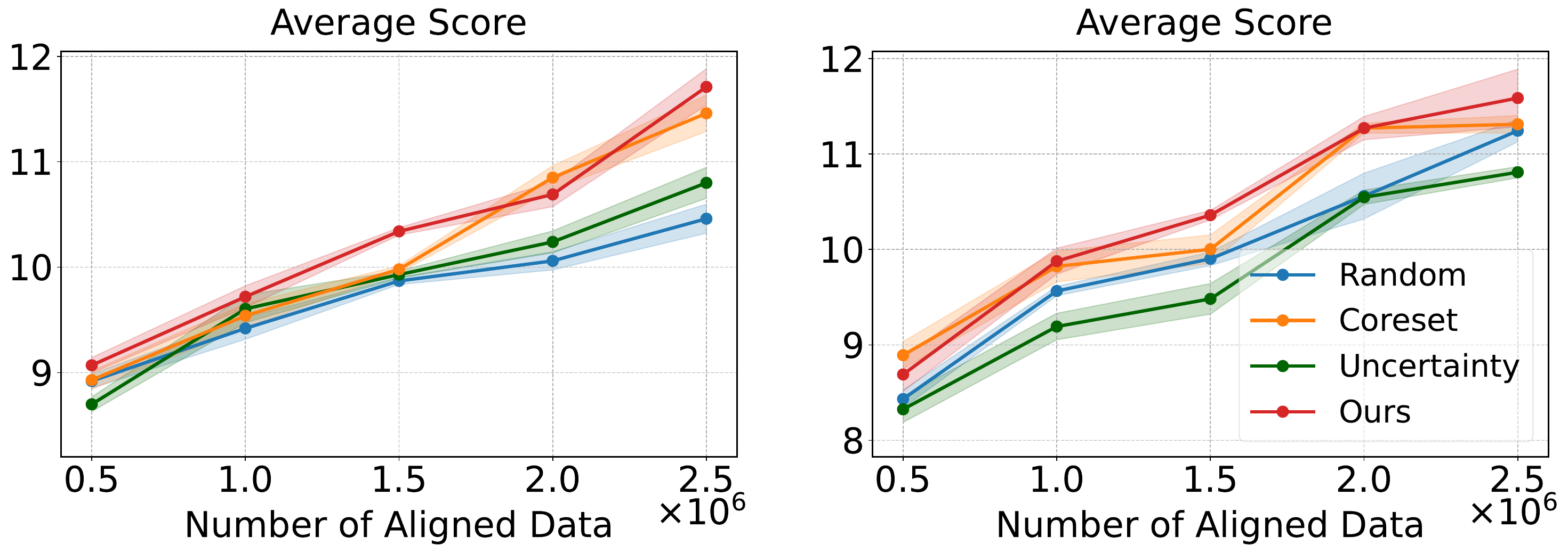}\hfill
\includegraphics[width=.33\linewidth]{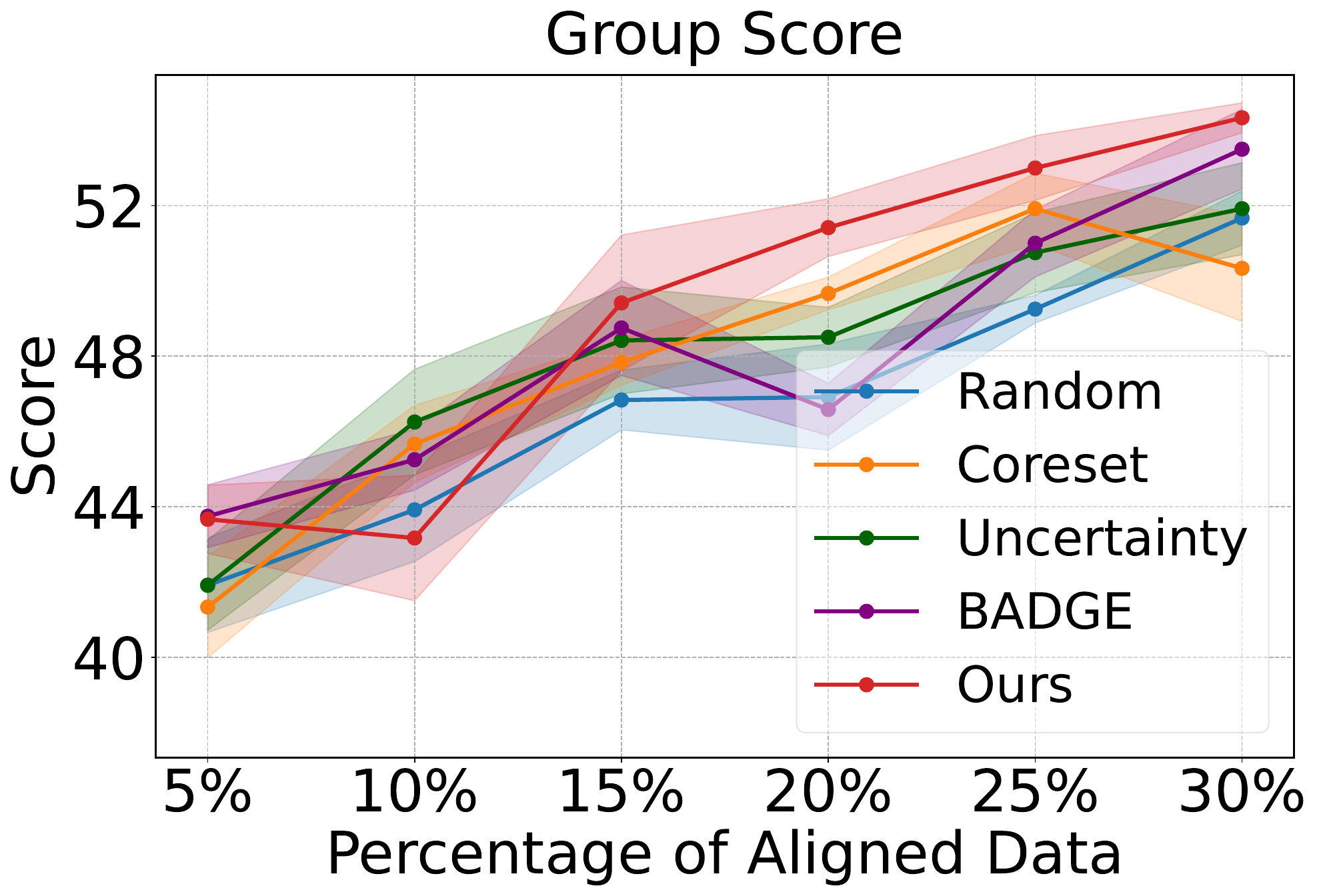}\hfill
\caption{
\emph{Left and  Middle:}
Results of streaming-based multimodal active learning on DataComp (CLIP-B32) under two streaming buffer sizes: 1024 (\textit{left}) and 2048 (\textit{middle}). We report the average score across 38 downstream tasks as learning progresses. 
\emph{Right:}
Results of pool-based multimodal active learning on ColorSwap (CLIP-B32): comparison with an adapted \badge baseline.
}
\label{fig:datacomp_diff_size}
\end{figure}

\end{document}